%% file: neurips_2026.tex
\documentclass{article}

\PassOptionsToPackage{numbers,compress}{natbib}

\usepackage[eandd,final]{neurips_2026}

\usepackage[utf8]{inputenc}
\usepackage[T1]{fontenc}
\usepackage{hyperref}
\usepackage{url}
\usepackage{booktabs}
\usepackage{amsfonts}
\usepackage{amsmath,amssymb,amsthm}
\usepackage{nicefrac}
\usepackage{microtype}
\usepackage{xcolor}
\usepackage{graphicx}
\usepackage{enumitem}
\usepackage{multirow}
\usepackage{makecell}
\usepackage{array}
\usepackage{placeins}
\hypersetup{
  hidelinks,
  pdftitle={SimplexUQ: An Evaluation Framework and Benchmark for Conformal Uncertainty on Simplex-Valued Predictions},
  pdfauthor={Liang You, Hengyu Shi, Dongwen Ou}
}

\newtheorem{theorem}{Theorem}
\newtheorem{proposition}[theorem]{Proposition}
\newtheorem{corollary}[theorem]{Corollary}
\newtheorem{lemma}[theorem]{Lemma}
\theoremstyle{definition}
\newtheorem{assumption}[theorem]{Assumption}
\newtheorem{remark}[theorem]{Remark}

\newcommand{\DeltaK}{\Delta^{K-1}}
\newcommand{\Pbb}{\mathbb{P}}

\newcommand{\cC}{\mathcal{C}}
\newcommand{\cD}{\mathcal{D}}
\newcommand{\cE}{\mathcal{E}}
\newcommand{\cX}{\mathcal{X}}
\newcommand{\cU}{\mathcal{U}}
\newcommand{\iid}{\stackrel{\mathrm{iid}}{\sim}}
\newcommand{\Unif}{\mathrm{Unif}}
\newcommand{\Exp}{\mathrm{Exp}}
\newcommand{\ceil}[1]{\left\lceil #1 \right\rceil}

\title{SimplexUQ: An Evaluation Framework and Benchmark for Conformal Uncertainty on Simplex-Valued Predictions}

\author{%
  Liang You\thanks{Corresponding author.} \\
  University of Pittsburgh \\
  \texttt{liangyou03@pitt.edu}
  \And
  Hengyu Shi \\
  Xiamen University
  \And
  Dongwen Ou \\
  Duke University
}

\begin{document}

\maketitle

\begin{abstract}
Conformal prediction guarantees marginal coverage, but a single calibration threshold can still spread that coverage unevenly, over-covering easy regions and under-covering hard ones. SimplexUQ is, to our knowledge, the first benchmark and reproducible protocol for measuring this allocation problem on simplex-valued predictions; it compares existing conformal wrappers rather than proposing a new one. Its task suite, SimplexTasks-12, combines six controlled synthetic regimes with six frozen-predictor real tasks spanning class probabilities, topic mixtures, spectral abundances, cell-type fractions, age distributions, and emotion mixtures. Each comparison fixes the predictor, score, and response-free stratification map, varies only the wrapper, and reports marginal coverage, worst-stratum coverage, max disparity, and within-task radius and compute. Global calibration can look valid while failing badly: on CIFAR-10 it attains 0.900 marginal coverage but only 0.542 in the worst entropy stratum, and Mondrian calibration raises that stratum to 0.886 while reducing max disparity from 0.358 to 0.022. No wrapper dominates, however. Under smooth synthetic heterogeneity, several repairs are competitive; fixed-map analyses show that rankings depend on the evaluation groups and protocol; and in a 12-task comparison, Mondrian has lower disparity on its single target partition for all 12 tasks, whereas BatchMVP has lower disparity over overlapping groups on five. These are empirical comparisons, not new coverage guarantees. A controlled predictor-bias sweep shows that removing predictor bias only partly reduces global-threshold disparity. We release task cards, result provenance, permitted derived arrays, and rebuild instructions, and treat wrapper selection as a diagnostic comparison rather than a universal ranking.
\end{abstract}

\section{Introduction}
\label{sec:intro}

Many predictors output compositions. Softmax class probabilities, topic proportions, spectral abundances, cell-type fractions, age-label distributions, and emotion mixtures all live on the simplex $\DeltaK = \{y \in [0,1]^K : \sum_k y_k = 1\}$. For these outputs, the practical question is not only whether a predictor is accurate, but whether its uncertainty estimates remain reliable in the hard regions of the simplex. Marginal coverage alone cannot answer this question.

Conformal prediction (CP) guarantees marginal coverage under exchangeability~\cite{vovk2005algorithmic,papadopoulos2002inductive}. On simplex-valued outputs, however, a single global threshold can still distribute that coverage unevenly across prediction space, over-covering easier regions and under-covering harder ones (Proposition~\ref{prop:split-disparity}). Coverage allocation, not marginal coverage alone, should therefore be a primary evaluation target.

Existing wrappers address heterogeneity, but the field lacks a simplex benchmark for choosing among them. The broader CP literature already offers grouped calibration, score normalization, exact or leave-one-out variants, and weighting-based adaptations~\cite{fontana2023conformal,barber2021predictive,tibshirani2019conformal}. What is missing is a benchmark that asks, under fixed simplex-valued predictors: which calibration mechanism works in which regime, what tradeoff it induces, and when no wrapper is an adequate repair.

We present \textbf{SimplexUQ}, an evaluation framework and benchmark for simplex-valued uncertainty. The gap between marginal and subgroup coverage is known and is our motivation, not our novelty claim; to our knowledge, SimplexUQ is the first benchmark that measures it on simplex-valued outputs with explicit allocation diagnostics. Our contributions are:
\begin{itemize}[leftmargin=1.5em,itemsep=1pt,topsep=2pt]
\item \textbf{Benchmark.} SimplexTasks-12: six controlled synthetic regimes and six frozen-predictor real tasks, released with task cards, permitted derived arrays, and rebuild scripts.
\item \textbf{Protocol.} Each comparison fixes the predictor, score, and response-free stratification map, varies only the wrapper, and reports marginal coverage, max disparity, worst-stratum coverage, and within-task radius.
\item \textbf{Theory.} Three propositions isolate the failure modes the benchmark tests: global-threshold misallocation, loss of split validity under same-sample normalization, and marginally valid but misallocated training-residual scaling.
\item \textbf{Findings.} Global CP keeps marginal coverage near 0.90 on all six real tasks, yet its worst stratum falls to 0.542 on CIFAR-10 and 0.410 on UTKFace. Under fixed maps, Mondrian has lower disparity than Global in 23 of 24 task--map cells, but no wrapper leads on every task and protocol, and BatchMVP has lower overlapping-family disparity than Mondrian on five tasks. These results yield a benchmark-derived selection heuristic.
\end{itemize}
All comparisons assume exchangeability and apply the same pre-declared map to calibration and test points; deployment claims still depend on the predictor, data source, and application.

\section{Related Work}
\label{sec:related}

\paragraph{Conformal wrappers for heterogeneous prediction.}
The closest line of work to ours is the family of conformal procedures designed to handle non-uniform difficulty across the input or prediction space. Split conformal prediction~\cite{vovk2005algorithmic,papadopoulos2002inductive} supports the marginal-validity guarantee, while full and jackknife+ variants~\cite{lei2018distribution,barber2021predictive,johnstone2025multioutput} trade compute for tighter or assumption-light alternatives. Heterogeneity-aware extensions take two routes: grouped (Mondrian) calibration partitions the prediction space into strata, and normalization-based methods rescale residuals by a local difficulty estimate. Classification-specific methods adapt set size to example difficulty, calibrate per class, or control false positives~\cite{angelopoulos2021uncertainty,ding2023class,fisch2022limited}. We do not propose a new wrapper; SimplexUQ benchmarks each of these as a wrapper family layered on a fixed simplex predictor, so that allocation behavior, not predictor quality, is what varies across rows.

\paragraph{Simplex-valued outputs and prediction-space allocation.}
Our setting borrows from two adjacent literatures. On the output side, compositional data analysis supplies log-ratio geometry~\cite{aitchison1982statistical}, and recent conformal methods build Dirichlet-based regions for compositional regression~\cite{amaral2025compositional} or credal sets over class distributions~\cite{javanmardi2024credal}; these constructions are typically tied to a single modeling pipeline. On the diagnostic side, conformal fairness work~\cite{vadlamani2025fairness,gao2025bridging,cresswell2025disparate} defines groups by protected attributes and studies coverage, or its downstream impact, across those groups. We answer a different question: we keep the predictor fixed and swap only the conformal wrapper, group test points by where they land on the simplex rather than by demographic attributes, and compare existing methods rather than propose new ones.

\paragraph{Evaluation as the artifact.}
Methodologically, this paper sits in the benchmark and evaluation-infrastructure tradition of WILDS, Dynabench, and HELM~\cite{koh2021wilds,kiela2021dynabench,liang2023holistic}, where the controlled comparison itself is the contribution. Existing conformal benchmarks emphasize scalar regression or classification, and to our knowledge none target simplex-valued outputs with explicit allocation diagnostics. SimplexUQ's contribution in this tradition is a standardized protocol, twelve tasks spanning controlled and real predictors, and a benchmark-derived wrapper-selection heuristic for simplex-valued uncertainty.

\section{Background and Setup}
\label{sec:background}

\paragraph{Setup and scope.}
Let $X \in \cX$ denote input features, $Y \in \DeltaK$ the target composition, and $U=\hat f(X)\in\DeltaK$ the output of a fixed predictor. A nonconformity score $R=d(Y,U)$ induces simplex-valued prediction sets. We use Aitchison distance for continuous compositional targets and total variation for classification-style tasks such as CIFAR-10, whose target is simplex-valued but not a natural continuous composition.

\paragraph{Split CP is the baseline.}
Given calibration residuals $\{R_i\}_{i=1}^{n_{\mathrm{cal}}}$, split conformal~\cite{vovk2005algorithmic,papadopoulos2002inductive} prediction constructs a threshold $\hat{q}$ from the empirical $(1-\alpha)$ quantile and predicts
\[
\cC_{\mathrm{split}}(X_{n+1}) = \{y \in \DeltaK : d(y, U_{n+1}) \le \hat{q}\}.
\]
Under exchangeability, this set attains marginal coverage at level $1-\alpha$. The key limitation is that the same threshold is applied to every prediction, even when the conditional distribution of $R \mid U$ varies sharply across the simplex.

We therefore focus on coverage allocation and wrapper selection. Score design, set geometry, and high-dimensional simplex geometry remain important, but they are held fixed here and lie outside the scope of the comparison.

\section{Strategy Taxonomy and Diagnostic Foundations}
\label{sec:taxonomy}

We compare wrappers by how they calibrate thresholds. Global methods use one threshold for all points. Group-wise methods replace that threshold with one per stratum. Normalization-based methods estimate a local scale before calibration. Approximate normalization and weighting serve as diagnostic cases: they show what happens when the split-conformal argument is broken or importance weights are unstable. This taxonomy is more informative than a flat leaderboard because each family targets a different heterogeneity pattern and fails for a different reason. Table~\ref{tab:wrappers} lists the wrappers used throughout the paper and the result that covers each one.

\begin{table}[t]
\centering
\small
\setlength{\tabcolsep}{3pt}
\caption{Wrappers compared in SimplexUQ. All operate on the same frozen predictor output $U$ and score $R=d(Y,U)$; they differ only in how the threshold is calibrated. Formal statements are in Appendix~A.}
\label{tab:wrappers}
\footnotesize
\begin{tabular}{@{}>{\raggedright\arraybackslash}p{0.14\linewidth}>{\raggedright\arraybackslash}p{0.45\linewidth}>{\raggedright\arraybackslash}p{0.38\linewidth}@{}}
\toprule
\textbf{Wrapper} & \textbf{Threshold rule} & \textbf{Guarantee or role} \\
\midrule
Global & One split-conformal threshold from all calibration scores & Marginal coverage (Theorem~\ref{thm:split}) \\
Mondrian & One threshold per stratum of the fixed map $g$ & Coverage within each stratum (Theorem~\ref{thm:mondrian}) \\
TwoStage & Local scale $\hat\sigma$ fitted on separate data; calibrate $R/\hat\sigma(U)$ & Marginal coverage (Theorem~\ref{thm:two-stage}); primary normalization wrapper \\
OneShot & kNN scale fitted on the same calibration residuals & Not guaranteed (Proposition~\ref{prop:oneshot}); diagnostic \\
TrainRes & Scale fitted on training residuals, then held fixed & Marginal, but can misallocate (Proposition~\ref{prop:trainres}); diagnostic \\
FullCP & Full-conformal local-scale variant & Exact reference; highest compute \\
Jackknife+ & Leave-one-out variant~\cite{barber2021predictive} & Leave-one-out reference \\
Weighted & Inverse local-scale weights, not density-ratio weights & Diagnostic failure case \\
BatchMVP & Multivalid patching over overlapping groups~\cite{jung2023batch} & No new guarantee claimed for our adaptation (Appendix~\ref{app:additional-multivalid}) \\
\bottomrule
\end{tabular}
\end{table}

\subsection{Three failure modes and main propositions}

Three propositions pin down the failure modes the benchmark targets; proofs and the standard split-validity results are in Appendix~A.

\begin{proposition}[Population-level global calibration under heterogeneous residuals]
\label{prop:split-disparity}
Fix $\alpha\in(0,1)$, $p\in(0,1)$ and $\lambda>1$ with $\alpha>p(1-1/\lambda)$. Consider an iid score model with binary prediction-space stratum $G\in\{A,B\}$ and
\[
\Pbb(G=B)=p,\qquad
R\mid G=A\sim \mathrm{Unif}[0,1],\qquad
R\mid G=B\sim \mathrm{Unif}[0,\lambda].
\]
Then the pooled population quantile $q_\star=(1-\alpha)/(1-p+p/\lambda)<1$ satisfies
\[
\Pbb(R\le q_\star)=1-\alpha
\]
while the stratum-wise coverages do not:
\[
\Pbb(R\le q_\star\mid G=A)>1-\alpha,
\qquad
\Pbb(R\le q_\star\mid G=B)<1-\alpha.
\]
Hence marginal calibration of a single global threshold does not imply simultaneous calibration in all prediction-space strata.
\end{proposition}

\begin{proposition}[Non-leave-one-out same-sample kNN normalization need not inherit split validity]
\label{prop:oneshot}
For every $n\ge 2$ and every $\alpha$ satisfying
\[
\frac{1}{n+1}\le \alpha < \frac{1}{2},
\]
there exists an iid calibration/test model and a one-shot local scale estimator $\hat\sigma_n$ fitted on the calibration sample. The same calibration residuals are used both to estimate local scales and to calibrate the empirical quantile. In this construction, all normalized calibration scores are identically equal to $1$, whereas the normalized test score has a different distribution and the resulting conformal set has coverage exactly
\[
\frac{1}{2}<1-\alpha.
\]
Therefore same-sample local-scale fitting and quantile calibration do not, by themselves, preserve the finite-sample split-conformal guarantee.
\end{proposition}

\begin{proposition}[Independent training-residual normalization is marginally valid but can misallocate]
\label{prop:trainres}
Let $\hat\sigma_{\mathrm{tr}}:\Delta^{K-1}\to(0,\infty)$ be a positive scale function fitted on training data independent of the calibration/test split and then held fixed during calibration. If calibration and test pairs are exchangeable conditional on that training data, then split conformal applied to the normalized score
\[
T=\frac{R}{\hat\sigma_{\mathrm{tr}}(U)}
\]
retains the usual marginal coverage guarantee. Nevertheless, for every $0<p<\alpha<1$ and every sufficiently small $\varepsilon\in(0,1)$, there exists a two-stratum problem with $\Pbb(G=B)=p$ and a fixed training-based scale field. In that problem, the pooled normalized rule is calibrated at level $1-\alpha$, but the hard-stratum coverage equals
\[
\frac{\varepsilon(1-\alpha)}
{2\left((1-p)+p\varepsilon/2\right)},
\]
which tends to $0$ as $\varepsilon\downarrow 0$. Therefore the failure mode is not loss of pooled marginal validity, but predictor-sensitive allocation collapse in hard regions.
\end{proposition}

In short, marginal calibration does not control allocation (Proposition~\ref{prop:split-disparity}), same-sample kNN normalization can lose split validity (Proposition~\ref{prop:oneshot}), and training-residual normalization can remain marginally valid while misallocating (Proposition~\ref{prop:trainres}).

\section{Evaluation Protocol}
\label{sec:protocol}

\paragraph{Stratification.}
To measure coverage allocation, we evaluate empirical coverage within prediction-space strata. Default strata are entropy- and boundary-based, with task-specific partitions when a natural grouping exists. Both the default rules and the appendix sensitivity sweeps are fixed before inspecting wrapper outcomes. Let $g:\DeltaK\to\{1,\dots,m\}$ denote the fixed, response-free stratification map (entropy bins, boundary bins, dominant component, or KMeans on $U$). For test indices $\mathcal T_s=\{i:g(U_i)=s\}$, let
\[
\hat c_s=\frac{1}{|\mathcal T_s|}\sum_{i\in\mathcal T_s}\mathbf{1}\{Y_i\in \hat C(X_i)\},
\qquad
\hat c=\frac{1}{n_{\mathrm{test}}}\sum_{i=1}^{n_{\mathrm{test}}}\mathbf{1}\{Y_i\in \hat C(X_i)\}.
\]
\paragraph{Metrics.} We report four metrics in the main text: marginal coverage $\hat c$ as a validity check, max disparity $\max_s|\hat c_s-(1-\alpha)|$ as the primary allocation diagnostic, worst-stratum coverage $\min_s \hat c_s$ as the failure floor, and mean radius $\bar r=n_{\mathrm{test}}^{-1}\sum_i r_i$ as an efficiency proxy. A wrapper is called \emph{valid} when its mean marginal coverage is at least $1-\alpha-\tau$ with the pre-declared tolerance $\tau=0.02$. Coverage variance, SSCV (size-stratified coverage violation), and the synthetic-only Monte Carlo simplex-volume ratio are reported in the appendix.

\paragraph{Two experiment series.} Results come from two series. The \emph{archived runs} compare all wrappers under task-specific configurations. Some of these configurations predate an explicit fixed-map flag, and a few cells come from auxiliary runs at a different scale or partition (marked $\dagger$). The \emph{fixed-map series} reruns Global, Mondrian, and BatchMVP on the released frozen arrays: each response-free map $g$ is computed once from the predictor outputs before splitting (synthetic draws receive their own map) and is then applied to both calibration and test points (Appendix~\ref{app:additional-evaluation}). We use the archived runs to describe all wrappers and the fixed-map series to check robustness, so we do not attribute every archived ranking solely to the wrapper. Computing $g$ from all predictor outputs is transductive and does not by itself establish the iid assumption of Theorem~\ref{thm:mondrian}; prospective use can fit maps on independent reference data.

\paragraph{Interpreting radius and validity.} Radius measures radial extent in the selected score, not simplex volume. At fixed predictor output and score, increasing the radius gives nested sets. Radius comparisons are therefore meaningful only within a task, where predictor, score, and simplex dimension are shared, and a smaller radius counts as an efficiency gain only when coverage and allocation are comparable. The empirical coverage tolerance above is a one-sided screening rule, not a finite-sample validity guarantee. Terminology and the scope of supplementary comparisons are detailed in Appendix~\ref{app:additional-evaluation}.

\begin{figure*}[t]
\centering
\includegraphics[width=0.96\textwidth]{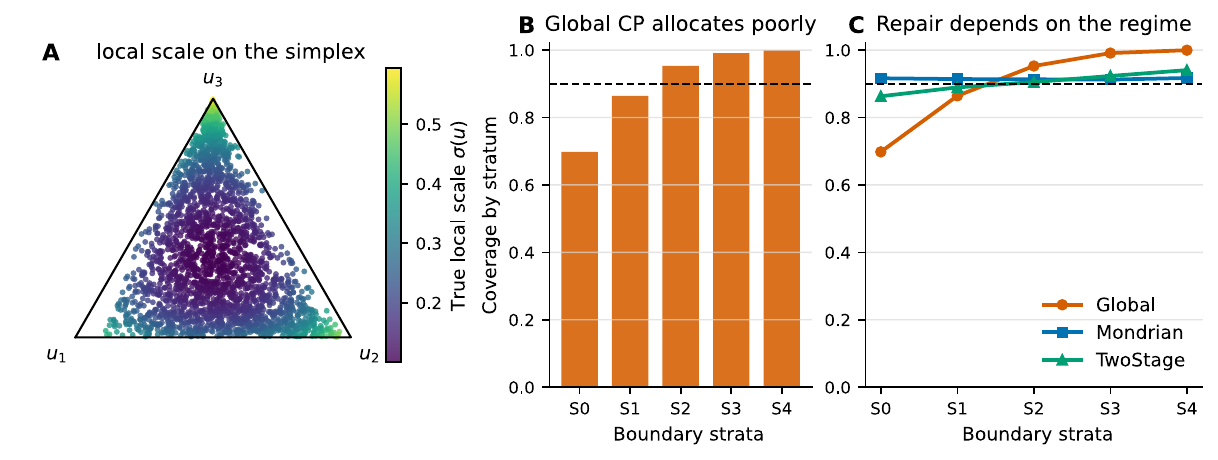}
\caption{Allocation failure on the D2 smooth-scale synthetic task. (A) The true local residual scale is highest near simplex vertices and lowest toward the interior. (B) Under the pre-declared boundary stratification, a single global split-conformal threshold under-covers boundary-adjacent strata despite near-nominal pooled coverage. (C) Both Mondrian and TwoStage improve allocation; this profile does not establish unique superiority of normalization.}
\label{fig:allocation-geometry}
\end{figure*}

\section{Synthetic Validation}
\label{sec:synthetic}

\paragraph{Synthetic regimes isolate heterogeneity structure.}
D1 is a homogeneous negative control. D2 and D5 isolate smooth scale heterogeneity, with D5 adding heavy tails. D3 isolates aligned discrete groups. D4 is a bias-type counterexample that is not reducible to a single local scale field. D6 raises simplex dimension to expose high-$K$ failure. Each regime isolates one mechanism, so when a wrapper succeeds or fails on a real task, the synthetic suite suggests a mechanistic explanation rather than a post-hoc label. Full synthetic specifications and sample sizes are in Appendix~\ref{app:synthetic-dgp}.

\paragraph{What the synthetic suite establishes.}
Figure~\ref{fig:allocation-geometry} illustrates heterogeneous coverage allocation. In Figure~\ref{fig:synthetic-heatmap}, D1 behaves as a control, both normalization and grouped calibration improve on Global in D2/D5, aligned D3 favors grouped calibration, and D4 remains only partially repairable. Paired intervals in Appendix~\ref{app:additional-synthetic-ci} do not support a uniquely superior normalization family: in D2, Mondrian outperforms TwoStage, while its difference from FullCP is not resolved by the interval. D6 has large global disparity; its daggered FullCP and Jackknife+ entries use a smaller exact-reference run and are not direct main-scale comparisons.

\begin{figure*}[t]
\centering
\includegraphics[width=0.9\textwidth]{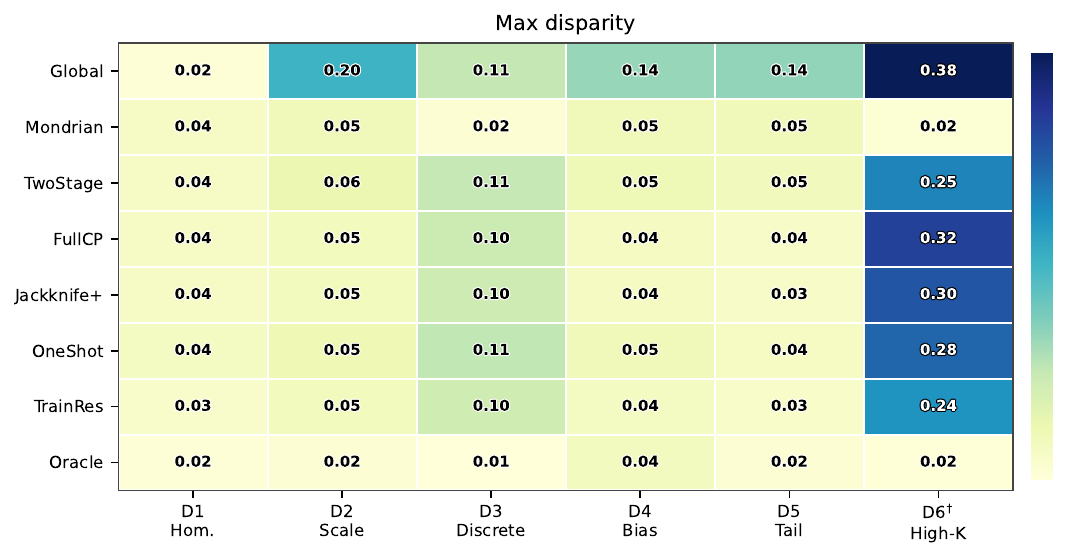}
\caption{Synthetic benchmark summary. Each cell reports max disparity. D2/D5 support several competitive repairs, aligned D3 favors group-wise calibration, and D6 shows large global disparity. For D6$^{\dagger}$, FullCP and Jackknife+ use a smaller exact-reference run. Paired comparisons are provided in Appendix~\ref{app:additional-synthetic-ci}.}
\label{fig:synthetic-heatmap}
\end{figure*}

\section{Real Benchmark Suite}
\label{sec:benchmark}

\paragraph{Benchmark design.}
The real-data suite spans six tasks across vision, NLP, remote sensing, genomics, and demographics. Each task fixes a realistic simplex-valued predictor and varies only the conformal wrapper, which isolates calibration behavior from predictor development.

\begin{table}[t]
\centering
\setlength{\tabcolsep}{3pt}
\caption{Real tasks in SimplexUQ.}
\label{tab:realtasks}
\begin{tabular}{@{}p{0.25\linewidth}p{0.14\linewidth}p{0.06\linewidth}p{0.12\linewidth}p{0.31\linewidth}@{}}
\toprule
\textbf{Task} & \textbf{Domain} & \textbf{$K$} & \textbf{$n$} & \textbf{Design hypothesis} \\
\midrule
CIFAR-10 softmax & Vision & 10 & 10000 & Classification-style grouped stress test \\
20 Newsgroups topics & NLP & 10 & 7539 & Mild / smooth heterogeneity \\
SemEval AffectiveText & NLP/Affect & 6 & 1000 & Small-sample weak structure \\
Samson NMF unmixing & Remote sensing & 3 & 9025 & Aligned grouped heterogeneity \\
PBMC deconvolution & Genomics & 8 & 5000 & Semi-synthetic sensitivity case \\
UTKFace age LDL & Demo-\newline graphics & 10 & 23687 & Strong grouped heterogeneity \\
\bottomrule
\end{tabular}
\end{table}

\paragraph{Task roles and artifact contract.}
CIFAR-10, Samson, and UTKFace provide grouped-repair examples; Topics and AffectiveText test the limits of that pattern; PBMC supplies a semi-synthetic sensitivity case. The tasks derive from CIFAR-10~\cite{krizhevsky2009learning}, 20 Newsgroups~\cite{lang1995newsweeder}, AffectiveText~\cite{strapparava2007semeval}, Samson~\cite{zhu2017spectral}, PBMC3K~\cite{zheng2017massively}, and UTKFace~\cite{zhang2017age}. With the synthetic suite they form \textbf{SimplexTasks-12}. The artifact includes task cards, permitted derived $(Y,U)$ arrays, and table/figure scripts; restricted assets instead have rebuild instructions and documented limitations.

\paragraph{Uncertainty reporting and compute.}
All real tasks use fixed cached predictors; the conformal benchmark is then repeated over random calibration/test splits. Most real tasks use 50 random splits, while AffectiveText and PBMC use 200 because the former is smaller and the latter is part of a semi-synthetic deconvolution stress test. Figure~\ref{fig:real-tradeoff} shows $\pm 1$ standard deviation. Evaluation is CPU-only once predictor outputs are cached (Appendix Table~\ref{tab:compute}); FullCP is the dominant cost, at about 3.25 seconds per repetition on full-scale CIFAR-10.

\section{Benchmark Results}
\label{sec:results}

\begin{figure*}[t]
\centering
\includegraphics[width=0.9\textwidth]{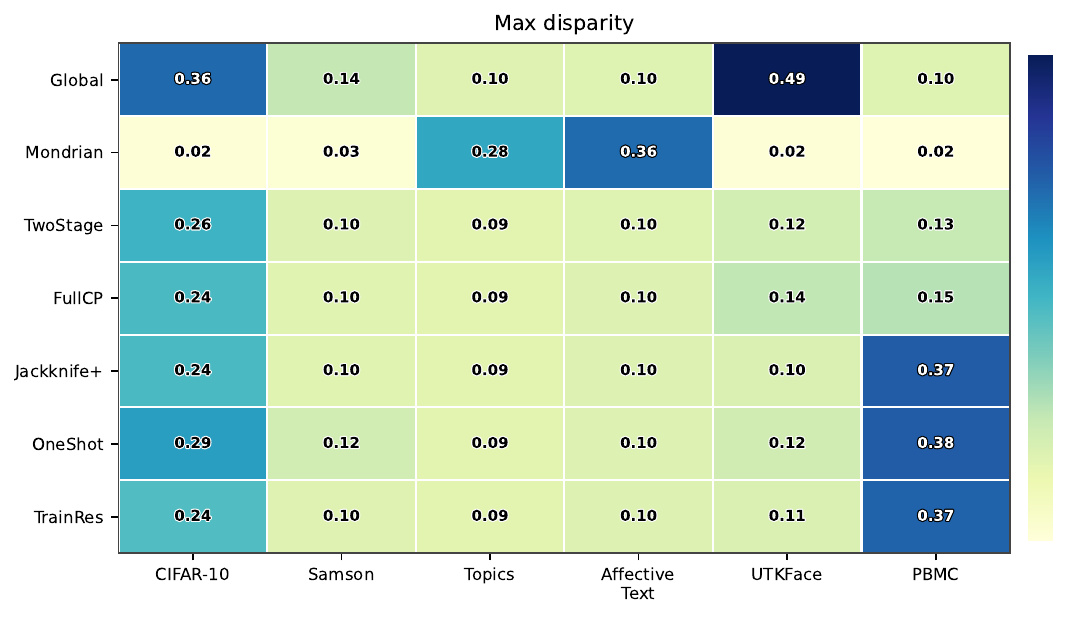}
\caption{Archived task-specific disparity results. CIFAR-10, Samson, and UTKFace favor grouped repair in these runs; Topics and AffectiveText differ under their archived protocols. Reference-only cells: UTKFace FullCP uses 2,000 samples; PBMC Jackknife+, OneShot, and TrainRes use five-stratum auxiliary runs, while its other displayed methods use two-stratum main runs. Those PBMC cells do not support cross-protocol rankings. Fixed-map robustness is reported separately in Appendix~\ref{app:additional-maps}.}
\label{fig:real-heatmap}
\end{figure*}

\begin{figure}[t]
\centering
\includegraphics[width=\linewidth]{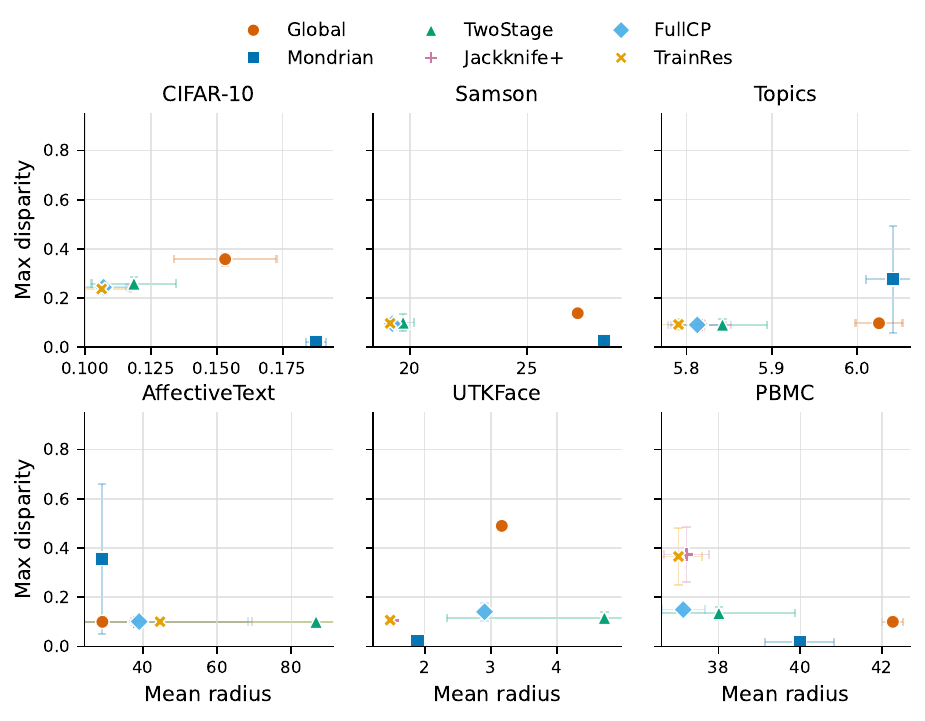}
\caption{Within-task disparity versus mean score radius from the same sources as the summary tables; bars show $\pm1$ standard deviation. Radius measures radial extent, not volume. Smaller values are preferable only alongside acceptable coverage and comparable protocols. UTKFace FullCP and PBMC auxiliary-method points retain the reference-run limitations in Figure~\ref{fig:real-heatmap}; axes are not comparable across tasks.}
\label{fig:real-tradeoff}
\end{figure}

\paragraph{Finding 1: global CP disparity is widespread on simplex tasks.}
On all six real tasks, global split CP attains 0.899--0.904 marginal coverage, yet its max disparity is at least 0.098 on every task and reaches 0.358 on CIFAR-10 and 0.490 on UTKFace, where the worst stratum has coverage 0.410 (Appendix Tables~\ref{tab:real-summary-1}--\ref{tab:real-summary-2}). The homogeneous control D1 is the expected exception. Allocation is therefore a primary evaluation target.

\paragraph{Finding 2: grouped repair is strong for a fixed partition, but no wrapper dominates.}
In the archived runs (Figures~\ref{fig:real-heatmap} and~\ref{fig:real-tradeoff}), Mondrian reduces max disparity from 0.358 to 0.022 on CIFAR-10, from 0.138 to 0.025 on Samson, and from 0.490 to 0.021 on UTKFace. On Topics and AffectiveText, however, it has the highest disparity among non-weighted wrappers (0.276 and 0.356), while normalization and leave-one-out variants are competitive on Topics. These runs do not separate heterogeneity structure from protocol effects.

The fixed-map series is more favorable to grouping. Mondrian has lower mean disparity than Global in 23 of 24 task--map cells, and 21 paired intervals exclude zero; the exception is Topics with entropy bins, which favors Global (Appendix Table~\ref{tab:additional-map-sweep}). An earlier all-wrapper analysis also ranks Mondrian first on five of six tasks under four maps, with AffectiveText mixed, but disagrees on Topics with entropy bins (Table~\ref{tab:real-strata-sensitivity}). Because Mondrian calibrates the same partition on which it is evaluated, its advantage is relative to the chosen map, and no single wrapper leads across all tasks and protocols. Its radius cost is also task-dependent: Mondrian's mean radius is 1.23 times Global's on CIFAR-10 but 0.60 times on UTKFace (Table~\ref{tab:radius-ratios}).

\paragraph{Finding 3: controls and counterexamples matter.}
D1 shows that adaptation is not always necessary. PBMC is a semi-synthetic sensitivity case rather than a uniformly homogeneous control: Mondrian improves allocation in its main run and across the four fixed maps, while the archived mixture-generation runs show different rankings (Appendix Figure~\ref{fig:appendix-pbmc-sensitivity}). Because those runs also differ in realized strata, they do not isolate a concentration-only effect. D4 and AffectiveText further show why task labels alone cannot determine the preferred wrapper. Consistent with Proposition~\ref{prop:oneshot}, OneShot has mean marginal coverage below the 0.90 target on all six real tasks (0.882--0.892). The inverse-scale Weighted rule fails the 0.88 coverage screen on every real task (0.130--0.770), so we treat it as a diagnostic failure case, not as a weighted-CP baseline.

\paragraph{Finding 4: a single partition and overlapping groups favor different methods.}
In the fixed-map series, we compare BatchMVP~\cite{jung2023batch}, Global, and Mondrian on all 12 tasks with paired splits. BatchMVP reduces default-partition disparity relative to Global on 10 tasks and overlapping-family disparity on 11. It has lower overlapping-family disparity than Mondrian on five tasks, whereas Mondrian has lower default-partition disparity on all 12. Most of the corresponding paired intervals exclude zero, although they are not multiplicity-adjusted (Appendix Table~\ref{tab:batch-ci}). The two methods therefore serve different targets: Mondrian a single partition, BatchMVP overlapping groups. BatchMVP passes the 0.88 coverage screen on every task but over-covers on CIFAR-10 (0.953) and Samson (0.962). Appendix~\ref{app:additional-multivalid} reports the full comparison, implementation limits, and paired intervals.

\section{Strategy-Selection Framework}
\label{sec:selection}

\paragraph{How to use the table.}
Table~\ref{tab:selection-rule} summarizes a heuristic derived from our benchmark results, not a validated optimal selection rule. It gives a recommendation only after a task's diagnostics have been computed under the pre-declared map. Coarse aligned heterogeneity points to grouped calibration; smooth heterogeneity calls for comparing normalization with grouping rather than choosing either automatically; overlapping target groups call for a multivalid comparison. Choose candidates on development data and confirm the final choice on held-out data. When no wrapper meets the application's allocation and radius criteria, the table gives no recommendation, and the next step is to investigate the predictor or score. The controlled predictor-bias experiment in Appendix~\ref{app:additional-predictor} supports this step but does not guarantee that predictor improvement repairs allocation. Recommendations are tied to the frozen predictor: after a material predictor change, the response-free diagnostics must be rerun.

\begin{table*}[!htbp]
\centering
\setlength{\tabcolsep}{3pt}
\caption{Benchmark-derived candidate-selection heuristic, to be applied after a task's diagnostics are computed. Recommendations are starting points for held-out comparison, not guarantees or a validated universal selection rule, and must be re-derived after a material predictor change.}
\label{tab:selection-rule}
\begin{tabular}{@{}>{\raggedright\arraybackslash}p{0.22\linewidth}>{\raggedright\arraybackslash}p{0.27\linewidth}>{\raggedright\arraybackslash}p{0.22\linewidth}>{\raggedright\arraybackslash}p{0.25\linewidth}@{}}
\toprule
\textbf{Observed pattern} & \textbf{Diagnostic evidence} & \textbf{Prefer} & \textbf{Main caveat} \\
\midrule
Near-homogeneous regime & Low max disparity and acceptable worst-stratum coverage & Global split CP & Extra adaptation can add variance or compute without improving allocation \\
Coarse aligned heterogeneity & Hard/easy regions line up with entropy bins, dominant class, or other meaningful groups & Group-wise / Mondrian CP & Fragmented or weakly aligned groups can become unstable in small samples \\
Smooth heterogeneity & Coverage changes gradually across entropy or boundary proximity & Compare two-stage normalization with Mondrian; exact variants if affordable & Naive OneShot normalization is not guaranteed, and exactness can be expensive \\
Bias-type or predictor-driven failure & High disparity remains after both grouping and normalization & No wrapper recommendation; investigate the predictor or score & Wrapper choice alone may not repair structural misspecification \\
Small-sample weak-structure setting & Few calibration points per group and similar floors across valid methods & Conservative global or smooth normalization strategies & Mondrian can overfit the calibration split and hurt worst-stratum behavior \\
\bottomrule
\end{tabular}
\end{table*}
\FloatBarrier

\section{Limitations, Future Work, and Conclusion}
\label{sec:limitations}
Results depend on predictor quality, score geometry, and the evaluation map. Scores are fixed within each task, not shared across tasks. Mondrian benefits from calibrating the partition on which it is evaluated; the overlapping-family analysis and map sweeps examine but do not remove this alignment. Some archived configurations lack explicit map provenance, and PBMC auxiliary and UTKFace exact-reference cells are not like-for-like main-run comparisons. We do not interpret their differences as pure method effects. Predictor outputs can also distort the diagnostic groups: low entropy does not establish accuracy, and the controlled bias sweep is not evidence that real-model retraining necessarily repairs coverage. A recommendation derived for one frozen predictor should therefore not be transferred to a materially different predictor without rerunning the response-free diagnostics. The theory identifies failure modes, not an optimal selection rule. Exact-method reference runs and the BatchMVP grid further limit quantitative generalization.

SimplexTasks-12 is built to be extended: (i) a \emph{score-function sweep} testing whether allocation patterns are score-invariant; (ii) \emph{shift-aware variants} of the real tasks with a likelihood-ratio weighted CP baseline instead of the diagnostic inverse-scale rule; (iii) \emph{high-$K$ scaling}, with language-model-scale tasks and exact or jackknife+ approximations beyond $K\!\sim\!50$; (iv) theory for \emph{when} each failure mode dominates, such as when normalization provably beats grouping; and (v) SimplexUQ as a training-time feedback signal for predictors. The released arrays, task cards, and rebuild scripts make these extensions inexpensive to attempt.

\paragraph{Conclusion.} Marginal coverage is not enough to evaluate conformal uncertainty on simplex-valued outputs: global CP stays near nominal coverage while its worst stratum can fall to 0.410. Which repair works depends on the heterogeneity pattern and the evaluation groups, so wrapper choice should rest on held-out diagnostic comparison, which SimplexTasks-12 makes reproducible.

\section{Funding and Competing Interests.}
This research received no external funding. The authors declare no competing interests.

\bibliographystyle{unsrtnat}
\FloatBarrier
\bibliography{refs}

\clearpage
\appendix

\section{Code and Data Availability}
\label{sec:availability}

The SimplexUQ code is available at
\url{https://github.com/liangyou03/SimplexUQ}.
The repository includes benchmark runners, experiment configurations,
tests, and figure-generation scripts.

The SimplexTasks-12 dataset artifact is hosted at
\url{https://huggingface.co/datasets/liangyou03/SimplexTasks-12-data}.
It contains task cards, frozen synthetic and derived arrays,
evaluation metadata, and archived result summaries.
Source-specific usage and redistribution restrictions are documented
in the artifact; restricted raw assets are not mirrored.

\section{Theory and Additional Experimental Details}

\subsection{Appendix roadmap}

\paragraph{Contents.}
\begin{itemize}[leftmargin=1.5em,itemsep=2pt]
\item Theory: proof conventions, standard split-validity results, proofs of Propositions~\ref{prop:split-disparity}--\ref{prop:trainres}, and an approximation result for local normalization.
\item Data: synthetic data-generating processes (Appendix~\ref{app:synthetic-dgp}), real-task preprocessing and predictors, and asset provenance (Table~\ref{tab:asset-provenance}).
\item Compute: omitted cells, runtime notes, and the compute environment (Table~\ref{tab:compute}).
\item Archived runs: stratified profiles, stratification sensitivity (Table~\ref{tab:real-strata-sensitivity}), PBMC sensitivity, and full real-task, per-stratum, and runtime tables (Tables~\ref{tab:real-summary-1}--\ref{tab:runtime-real}).
\item Fixed-map series and supplementary comparisons (Appendix~\ref{app:additional-evaluation}): terminology, radius ratios, paired synthetic intervals, BatchMVP, map and count sweeps, and the controlled predictor-quality experiment.
\end{itemize}

\paragraph{Reproducibility.}
The accompanying supplement contains the source-result snapshots, SHA256 manifests, and the original plotting template. Its audit script checks raw repetition metrics against summaries and rebuilds the numeric tables from the same sources as the figures. The SimplexTasks-12 artifact supplies task cards, configurations, permitted derived arrays, and rebuild instructions for restricted assets. Current released runners use fixed response-free maps; this does not retroactively certify the protocol of older result files. AffectiveText evaluation uses frozen outputs, and its open fallback constructs a different predictor rather than reproducing the closed-model cache.

\subsection{Formal proof setup}

\paragraph{Score-level reduction and quantile convention.}
All proofs are stated at the scalar nonconformity-score level. Let $Z=(X,Y)$ be an observation, where $Y\in\DeltaK$ is the target and $U=\hat f(X)\in\DeltaK$ is the output of a fixed predictor. Let
\[
R=d(Y,U)\in[0,\infty)
\]
be a task-defined nonconformity score. For a calibration sample of size $n$, define the split-conformal quantile index
\[
k_{n,\alpha}=\ceil{(n+1)(1-\alpha)}.
\]
If $k_{n,\alpha}\le n$, the conformal threshold is the $k_{n,\alpha}$-th smallest calibration score; otherwise the threshold is set to $+\infty$. This is the standard conservative split-conformal convention.

\begin{remark}[Conditioning on fitted objects]
All validity statements below are conditional on any predictor-fit sample, fixed stratification map, or fixed scale estimator used before conformal calibration. After conditioning on those fitted objects, the calibration and test observations must remain exchangeable.
\end{remark}

\subsection{Standard split-validity results used in the paper}

\begin{lemma}[Rank lemma]\label{lem:rank}
Let $S_1,\dots,S_n,S_{n+1}$ be exchangeable real-valued scores and let $\hat q$ be the split-conformal threshold computed from $S_1,\dots,S_n$ at level $1-\alpha$ using the convention above. Then
\[
\Pbb\{S_{n+1}\le \hat q\}\ge 1-\alpha.
\]
If the scores are almost surely distinct and $k_{n,\alpha}\le n$, then
\[
\Pbb\{S_{n+1}\le \hat q\}=\frac{k_{n,\alpha}}{n+1}.
\]
\end{lemma}

\begin{proof}
If $k_{n,\alpha}>n$, then $\hat q=+\infty$ and the claim is immediate. Assume $k=k_{n,\alpha}\le n$. Break ties by adding iid infinitesimal continuous tie-breakers independent of the scores, and let $\widetilde{\operatorname{rank}}(S_{n+1})$ be the resulting rank of the test score among $S_1,\dots,S_n,S_{n+1}$. By exchangeability, this randomized rank is uniform on $\{1,\dots,n+1\}$. If $\widetilde{\operatorname{rank}}(S_{n+1})\le k$, then the $k$-th calibration order statistic is at least $S_{n+1}$, hence $S_{n+1}\le \hat q$. Therefore
\[
\Pbb\{S_{n+1}\le \hat q\}
\ge
\Pbb\{\widetilde{\operatorname{rank}}(S_{n+1})\le k\}
=\frac{k}{n+1}
\ge 1-\alpha.
\]
If the scores are almost surely distinct, then $\widetilde{\operatorname{rank}}(S_{n+1})$ equals the ordinary rank and the event $S_{n+1}\le \hat q$ is equivalent to $\widetilde{\operatorname{rank}}(S_{n+1})\le k$, which gives the equality statement.
\end{proof}

\begin{theorem}[Split conformal prediction on the simplex]\label{thm:split}
Assume the calibration pairs $(U_i,R_i)_{i=1}^n$ and the test pair $(U_{n+1},R_{n+1})$ are exchangeable after conditioning on the fitted predictor $\hat f$. Let $\hat q$ be the split-conformal threshold computed from $R_1,\dots,R_n$, and define
\[
\cC_{\mathrm{split}}(x)=
\left\{y\in\DeltaK:
d\bigl(y,\hat f(x)\bigr)\le \hat q
\right\}.
\]
Then
\[
\Pbb\{Y_{n+1}\in \cC_{\mathrm{split}}(X_{n+1})\}\ge 1-\alpha.
\]
\end{theorem}

\begin{proof}
The event $Y_{n+1}\in \cC_{\mathrm{split}}(X_{n+1})$ is exactly $R_{n+1}\le \hat q$. Apply Lemma~\ref{lem:rank} to the exchangeable scores $R_1,\dots,R_n,R_{n+1}$.
\end{proof}

\begin{theorem}[Two-stage normalized split conformal validity]\label{thm:two-stage}
Let $\cD_{\mathrm{fit}}$ be data used to fit $\hat f$ and let $\cD_{\mathrm{scale}}$ be data used to fit a positive scale function $\hat\sigma:\DeltaK\to(0,\infty)$. Assume that, conditional on $\cD_{\mathrm{fit}}$ and $\cD_{\mathrm{scale}}$, the calibration pairs $(U_i,R_i)_{i=1}^n$ and the test pair $(U_{n+1},R_{n+1})$ are exchangeable. Define
\[
T_i=\frac{R_i}{\hat\sigma(U_i)},\qquad i=1,\dots,n,
\]
let $\hat q_T$ be the split-conformal threshold of $T_1,\dots,T_n$, and define
\[
\cC_{\mathrm{norm}}(x)=
\left\{y\in\DeltaK:
\frac{d\bigl(y,\hat f(x)\bigr)}{\hat\sigma(\hat f(x))}
\le \hat q_T
\right\}.
\]
Then
\[
\Pbb\{Y_{n+1}\in \cC_{\mathrm{norm}}(X_{n+1})\mid \cD_{\mathrm{fit}},\cD_{\mathrm{scale}}\}\ge 1-\alpha,
\]
and therefore
\[
\Pbb\{Y_{n+1}\in \cC_{\mathrm{norm}}(X_{n+1})\}\ge 1-\alpha.
\]
\end{theorem}

\begin{proof}
Conditional on $\cD_{\mathrm{fit}}$ and $\cD_{\mathrm{scale}}$, the function $u\mapsto\hat\sigma(u)$ is fixed. Hence the transformed scores $T_i=R_i/\hat\sigma(U_i)$, $i=1,\dots,n+1$, are exchangeable because they are obtained from exchangeable pairs $(U_i,R_i)$ by applying the same fixed measurable map to each pair. The event $Y_{n+1}\in \cC_{\mathrm{norm}}(X_{n+1})$ is exactly $T_{n+1}\le \hat q_T$. Lemma~\ref{lem:rank} gives the conditional coverage result, and averaging over $\cD_{\mathrm{fit}}$ and $\cD_{\mathrm{scale}}$ gives the unconditional result.
\end{proof}

\begin{remark}[What Theorem~\ref{thm:two-stage} does and does not say]
The theorem proves marginal validity for the normalized score. It does not prove conditional validity given $U=u$ or given a fine prediction-space stratum. The empirical question is whether the fitted scale function improves allocation and efficiency relative to global calibration.
\end{remark}

\begin{theorem}[Fixed group-wise split conformal validity]\label{thm:mondrian}

Let $g:\DeltaK\to\{1,\dots,m\}$ be a fixed stratification map that does not depend on the calibration or test responses, and write $G_i=g(U_i)$. The map may be specified analytically or learned from reference predictor outputs independent of the calibration/test pairs. Response-freeness alone is insufficient: the validity statement below requires the calibration and test pairs to remain iid after conditioning on $g$. For each group $s$, let

\[
\mathcal{I}_s=\{i\in\{1,\dots,n\}:G_i=s\},\qquad n_s=|\mathcal{I}_s|.
\]
Let $\hat q_s$ be the split-conformal threshold computed from $\{R_i:i\in\mathcal{I}_s\}$ at level $1-\alpha$, with $\hat q_s=+\infty$ if the group has too few calibration points for a finite threshold. Define
\[
\cC_{\mathrm{grp}}(x)=
\left\{y\in\DeltaK:
d\bigl(y,\hat f(x)\bigr)\le \hat q_{g(\hat f(x))}
\right\}.
\]
If the calibration and test pairs are iid after conditioning on the fitted predictor and on the fixed map $g$, then for each group $s$ with $\Pbb(G=s)>0$,
\[
\Pbb\{Y_{n+1}\in\cC_{\mathrm{grp}}(X_{n+1})\mid G_{n+1}=s\}\ge 1-\alpha.
\]
Consequently,
\[
\Pbb\{Y_{n+1}\in\cC_{\mathrm{grp}}(X_{n+1})\}\ge 1-\alpha.
\]
\end{theorem}

\begin{proof}
Fix a group $s$ with positive probability and condition on the event $G_{n+1}=s$ and on the calibration index set $\mathcal I_s$. Conditional on this event, the variables $\{R_i:i\in\mathcal I_s\}$ and $R_{n+1}$ are iid from the conditional law of $R\mid G=s$, hence exchangeable. The threshold $\hat q_s$ is exactly the split-conformal threshold computed from the calibration scores in this conditional sample, with sample size $n_s=|\mathcal I_s|$. Applying Lemma~\ref{lem:rank} within group $s$ yields
\[
\Pbb\{R_{n+1}\le \hat q_s\mid G_{n+1}=s,\mathcal I_s\}\ge 1-\alpha.
\]
Averaging over $\mathcal I_s$ gives the conditional group guarantee, and averaging over $G_{n+1}$ gives the unconditional guarantee.
\end{proof}

\begin{remark}[Finite-sample fragmentation]
Theorem~\ref{thm:mondrian} proves validity, not efficiency. When $n_s$ is small, the group-specific threshold can become large or infinite, and the resulting sets can be unstable or inefficient. This is the formal version of the benchmark's finite-sample fragmentation caveat.
\end{remark}

\subsection{Formal proofs for the three main propositions}

\paragraph{Proof of Proposition~\ref{prop:split-disparity}.} We restate the proposition in self-contained form and prove it. For example, $p=\alpha/2$ and $\lambda=2$ satisfy its condition, since $p(1-1/\lambda)=\alpha/4<\alpha$.
\begin{proposition}[Constructive form of Proposition~\ref{prop:split-disparity}]\label{prop:global-disparity}
Let $G\in\{A,B\}$ be a prediction-space stratum with
\[
\Pbb(G=B)=p,\qquad \Pbb(G=A)=1-p.
\]
For some $\lambda>1$, suppose
\[
R\mid G=A\sim \Unif[0,1],\qquad
R\mid G=B\sim \Unif[0,\lambda].
\]
Fix $\alpha\in(0,1)$ and assume
\[
\alpha>p\left(1-\frac{1}{\lambda}\right).
\]
Then there exists a single global threshold $q_\star$ such that
\[
\Pbb(R\le q_\star)=1-\alpha,
\]
but the stratum-wise coverages are unequal:
\[
\Pbb(R\le q_\star\mid G=A)>1-\alpha,\qquad
\Pbb(R\le q_\star\mid G=B)<1-\alpha.
\]
\end{proposition}

\begin{proof}
For $q\in[0,1]$, the pooled residual cdf is
\[
F(q)=\Pbb(R\le q)
=(1-p)q+p\frac{q}{\lambda}
=q\left(1-p+\frac{p}{\lambda}\right).
\]
Set
\[
q_\star=\frac{1-\alpha}{1-p+p/\lambda}.
\]
The condition $\alpha>p(1-1/\lambda)$ is equivalent to
\[
1-\alpha<1-p+\frac{p}{\lambda},
\]
so $q_\star<1$. Therefore $q_\star\in[0,1]$ and $F(q_\star)=1-\alpha$. The stratum-wise coverages are
\[
\Pbb(R\le q_\star\mid G=A)=q_\star
=\frac{1-\alpha}{1-p+p/\lambda}
>1-\alpha,
\]
because $1-p+p/\lambda<1$, and
\[
\Pbb(R\le q_\star\mid G=B)=\frac{q_\star}{\lambda}
=\frac{1-\alpha}{\lambda(1-p+p/\lambda)}
<1-\alpha,
\]
because $\lambda(1-p+p/\lambda)=p+\lambda(1-p)>1$.
\end{proof}

\begin{corollary}[Mixture validity does not imply stratum validity]\label{cor:mixture}
Let $S_1,\dots,S_m$ be prediction-space strata with weights $\pi_s>0$ satisfying $\sum_{s=1}^m\pi_s=1$ and conditional residual cdfs $F_s$. If a global threshold $q$ satisfies
\[
\sum_{s=1}^m \pi_s F_s(q)=1-\alpha
\]
and the values $F_s(q)$ are not all equal, then at least one stratum under-covers and at least one stratum over-covers relative to $1-\alpha$.
\end{corollary}

\begin{proof}
If all $F_s(q)$ were at least $1-\alpha$ and at least one were strictly larger, the weighted average would be strictly larger than $1-\alpha$. If all were at most $1-\alpha$ and at least one were strictly smaller, the weighted average would be strictly smaller. Since the weighted average equals $1-\alpha$ and the values are not all equal, both signs of deviation must occur.
\end{proof}

\paragraph{Proposition~\ref{prop:oneshot} is proved by the following counterexample.}
\begin{proposition}[Counterexample for non-leave-one-out one-shot normalization]\label{prop:oneshot-counterexample}
For every $n\ge 2$ and every nominal level $\alpha$ satisfying
\[
\frac{1}{n+1}\le \alpha<\frac{1}{2},
\]
there exists an iid calibration/test distribution and a one-shot local scale estimator fitted on the calibration residuals such that all normalized calibration scores are identically equal to $1$, the normalized test score has a different distribution, and the resulting conformal set has coverage exactly $1/2<1-\alpha$.
\end{proposition}

\begin{proof}
Let
\[
U_i\iid \Unif[0,1],\qquad R_i\iid \Exp(1),\qquad i=1,\dots,n+1,
\]
with all $U_i$ independent of all $R_i$. Given the calibration sample, define the nearest-neighbor scale estimator
\[
\hat\sigma_n(u)=R_{J(u)},\qquad
J(u)\in\arg\min_{1\le j\le n}|u-U_j|,
\]
with deterministic tie-breaking. Since the $U_i$ are continuous, ties occur with probability zero. For every calibration point $U_i$, its nearest calibration point is itself, so
\[
\hat\sigma_n(U_i)=R_i,\qquad
T_i=\frac{R_i}{\hat\sigma_n(U_i)}=1,\qquad i=1,\dots,n.
\]
Thus the empirical conformal threshold computed from $T_1,\dots,T_n$ is $\hat q_T=1$ for the stated range of $\alpha$.

For the test point, let $J=J(U_{n+1})$. Conditional on $U_1,\dots,U_n,U_{n+1}$, the index $J$ is fixed and $R_J$ is independent of $R_{n+1}$ with the same exponential distribution. Hence
\[
T_{n+1}=\frac{R_{n+1}}{\hat\sigma_n(U_{n+1})}
=\frac{R_{n+1}}{R_J},
\]
and therefore
\[
\Pbb(T_{n+1}\le 1)=\Pbb(R_{n+1}\le R_J)=\frac{1}{2}.
\]
The coverage of the normalized split set is thus
\[
\Pbb(T_{n+1}\le \hat q_T)=\frac{1}{2}<1-\alpha.
\]
Moreover, $T_i=1$ almost surely for each calibration point, whereas $T_{n+1}$ is not identically $1$. The calibration normalized scores and the test normalized score are therefore not exchangeable.
\end{proof}

\begin{remark}[Interpretation for the benchmark]
This does not prove that every OneShot implementation fails. It proves that the standard split-conformal proof cannot be invoked merely because a quantile is computed from normalized calibration scores that were themselves constructed using the same residuals.
\end{remark}

\paragraph{Proposition~\ref{prop:trainres} separates marginal validity from allocation failure.}
\begin{theorem}[Fixed training-residual scale still gives marginal validity]\label{thm:trainres-marginal}
Let $\hat\sigma_{\mathrm{tr}}:\DeltaK\to(0,\infty)$ be any positive scale function fitted only from training data that are independent of the calibration and test observations. Conditional on the training data, define
\[
T_i=\frac{R_i}{\hat\sigma_{\mathrm{tr}}(U_i)},\qquad i=1,\dots,n,
\]
let $\hat q_T$ be the split-conformal threshold from $T_1,\dots,T_n$, and define
\[
\cC_{\mathrm{tr}}(x)=
\left\{y\in\DeltaK:
d\bigl(y,\hat f(x)\bigr)\le \hat q_T\hat\sigma_{\mathrm{tr}}(\hat f(x))
\right\}.
\]
If the calibration and test observations are exchangeable conditional on the training data, then
\[
\Pbb\{Y_{n+1}\in\cC_{\mathrm{tr}}(X_{n+1})\}\ge 1-\alpha.
\]
\end{theorem}

\begin{proof}
Conditional on the training data, $\hat\sigma_{\mathrm{tr}}$ is fixed. Therefore the normalized scores $T_i=R_i/\hat\sigma_{\mathrm{tr}}(U_i)$, $i=1,\dots,n+1$, are exchangeable. The coverage event is $T_{n+1}\le \hat q_T$. Lemma~\ref{lem:rank} gives the result.
\end{proof}

\begin{proposition}[Hard-stratum failure from overfit training scales]\label{prop:trainres-counterexample}
Fix $0<p<\alpha<1$ and $\varepsilon\in(0,1)$. There exists a two-stratum problem and a training-residual scale estimate such that the pooled normalized-score threshold is exactly calibrated at level $1-\alpha$ in the population, but the hard-stratum coverage equals
\[
\frac{\varepsilon(1-\alpha)}{2\bigl((1-p)+p\varepsilon/2\bigr)},
\]
which tends to $0$ as $\varepsilon\downarrow 0$.
\end{proposition}

\begin{proof}
Let $G\in\{A,B\}$ with $\Pbb(G=B)=p$ and $\Pbb(G=A)=1-p$. Let the fresh residual distributions be
\[
R\mid G=A\sim \Unif[0,1],\qquad
R\mid G=B\sim \Unif[0,2].
\]
Suppose the training residuals overfit the hard region $B$, so that the training-based scale estimate is
\[
\hat\sigma_{\mathrm{tr}}(A)=1,\qquad
\hat\sigma_{\mathrm{tr}}(B)=\varepsilon.
\]
The normalized score is $T=R/\hat\sigma_{\mathrm{tr}}(G)$. Hence, for $q\in[0,1]$,
\[
\Pbb(T\le q)
=(1-p)q+p\Pbb\{R/\varepsilon\le q\mid G=B\}
=q\left((1-p)+\frac{p\varepsilon}{2}\right).
\]
Define the population normalized threshold
\[
q_\star=\frac{1-\alpha}{(1-p)+p\varepsilon/2}.
\]
Since $p<\alpha$, we have
\[
1-\alpha<1-p<(1-p)+\frac{p\varepsilon}{2},
\]
so $q_\star\in(0,1)$ for every $\varepsilon\in(0,1)$, and then
\[
\Pbb(T\le q_\star)=1-\alpha.
\]
The hard-stratum prediction radius is $q_\star\varepsilon$, so the hard-stratum coverage is
\[
\Pbb(R\le q_\star\varepsilon\mid G=B)
=\frac{q_\star\varepsilon}{2}
=\frac{\varepsilon(1-\alpha)}{2\bigl((1-p)+p\varepsilon/2\bigr)}.
\]
This quantity tends to $0$ as $\varepsilon\downarrow 0$.
\end{proof}

\begin{remark}[Safe interpretation of Proposition~\ref{prop:trainres}]
Training-residual normalization should not be described as lacking a split-validity argument when the scale function is fitted on an independent training sample and then held fixed. The correct statement is that it can remain marginally valid while still producing severe hard-region misallocation when the training scale field is distorted by predictor overfit.
\end{remark}

\subsection{Approximation result: when local normalization reduces disparity}

\begin{assumption}[Scale model]\label{ass:scale}
There exists a positive scale function $\sigma:\cU\to(0,\infty)$ such that, conditional on $U=u$,
\[
R=\sigma(u)\xi_u,
\]
where $\xi_u$ is a normalized residual with conditional cdf
\[
F_u(t)=\Pbb(\xi_u\le t\mid U=u).
\]
\end{assumption}

\begin{assumption}[Normalized-tail stability]\label{ass:tail}
There is a reference cdf $F_\star$, continuous at a point $q_\star\ge 0$ satisfying $F_\star(q_\star)=1-\alpha$, and a number $\rho\ge 0$ such that, for all $u\in\cU$ and all $t$ in a neighborhood $I$ of $q_\star$,
\[
|F_u(t)-F_\star(t)|\le \rho.
\]
Assume also that every $F_u$ is locally Lipschitz on $I$ with common constant $L$.
\end{assumption}

\begin{assumption}[Scale and quantile accuracy]\label{ass:accuracy}
For numbers $\varepsilon,\eta\ge 0$, define the event
\[
\cE_{\varepsilon,\eta}
=
\left\{
\sup_{u\in\cU}\left|\frac{\hat\sigma(u)}{\sigma(u)}-1\right|\le \varepsilon,
\quad
|\hat q-q_\star|\le \eta
\right\}.
\]
Assume that on this event the interval between $q_\star$ and
\[
\hat q\frac{\hat\sigma(u)}{\sigma(u)}
\]
lies inside $I$ for every $u\in\cU$.
\end{assumption}

\begin{theorem}[Local coverage error under accurate normalization]\label{thm:approx-local}
Under Assumptions~\ref{ass:scale}--\ref{ass:accuracy}, assume additionally that $(\hat q,\hat\sigma)$ are fitted from data independent of the present test point, so conditional on $(\hat q,\hat\sigma)$ the normalized residual $\xi_u=R/\sigma(u)$ given $U=u$ retains cdf $F_u$. Then, on the event $\cE_{\varepsilon,\eta}$, the locally normalized prediction set
\[
\cC_{\mathrm{loc}}(x)
=
\left\{y\in\DeltaK:
d\bigl(y,\hat f(x)\bigr)\le \hat q\hat\sigma(\hat f(x))
\right\}
\]
satisfies, for every $u\in\cU$,
\[
\left|
\Pbb\{Y\in\cC_{\mathrm{loc}}(X)\mid U=u,\hat q,\hat\sigma\}
-(1-\alpha)
\right|
\le
\rho+L\left[\eta+(q_\star+\eta)\varepsilon\right].
\]
\end{theorem}

\begin{proof}
Fix $u\in\cU$ and work on $\cE_{\varepsilon,\eta}$. Define
\[
t_u=\hat q\frac{\hat\sigma(u)}{\sigma(u)}.
\]
By the scale model,
\[
\Pbb\{Y\in\cC_{\mathrm{loc}}(X)\mid U=u,\hat q,\hat\sigma\}
=
\Pbb\{R\le \hat q\hat\sigma(u)\mid U=u,\hat q,\hat\sigma\}
=
F_u(t_u).
\]
Moreover,
\begin{align*}
|t_u-q_\star|
&\le \left|\hat q\frac{\hat\sigma(u)}{\sigma(u)}-\hat q\right|+|\hat q-q_\star| \\
&\le |\hat q|\left|\frac{\hat\sigma(u)}{\sigma(u)}-1\right|+\eta \\
&\le (q_\star+\eta)\varepsilon+\eta.
\end{align*}
By Assumption~\ref{ass:accuracy}, $t_u\in I$. Therefore,
\begin{align*}
|F_u(t_u)-(1-\alpha)|
&=|F_u(t_u)-F_\star(q_\star)| \\
&\le |F_u(t_u)-F_u(q_\star)|+|F_u(q_\star)-F_\star(q_\star)| \\
&\le L|t_u-q_\star|+\rho \\
&\le \rho+L\left[\eta+(q_\star+\eta)\varepsilon\right].
\end{align*}
\end{proof}

\begin{corollary}[Disparity bound for accurate local normalization]\label{cor:local-disparity}
Under the assumptions of Theorem~\ref{thm:approx-local}, on $\cE_{\varepsilon,\eta}$,
\[
\begin{aligned}
&\sup_{u,v\in\cU}
\left|
\Pbb\{Y\in\cC_{\mathrm{loc}}(X)\mid U=u,\hat q,\hat\sigma\}
-
\Pbb\{Y\in\cC_{\mathrm{loc}}(X)\mid U=v,\hat q,\hat\sigma\}
\right|
\\ &\qquad\le
2\rho+2L\left[\eta+(q_\star+\eta)\varepsilon\right].
\end{aligned}
\]
\end{corollary}

\begin{proof}
Apply Theorem~\ref{thm:approx-local} once at $u$, once at $v$, and use the triangle inequality.
\end{proof}

\subsection{Claims supported by this appendix}

\begin{itemize}[leftmargin=1.5em]
\item Standard split CP is marginally valid under exchangeability, by Theorem~\ref{thm:split}.
\item Global CP can be marginally valid and badly misallocated across simplex strata, by Proposition~\ref{prop:global-disparity} and Corollary~\ref{cor:mixture}.
\item Two-stage normalization is marginally valid when the scale function is fitted on a separate sample and held fixed during calibration, by Theorem~\ref{thm:two-stage}.
\item Fixed group-wise calibration is conditionally valid within the chosen groups, by Theorem~\ref{thm:mondrian}.
\item OneShot-style normalization is not automatically valid, by counterexample in Proposition~\ref{prop:oneshot-counterexample}.
\item Training-residual normalization can remain marginally valid while still collapsing on hard strata under overfit scales, by Theorem~\ref{thm:trainres-marginal} and Proposition~\ref{prop:trainres-counterexample}.
\item Local normalization reduces disparity when normalized tails are stable and scale estimates are accurate, by Theorem~\ref{thm:approx-local} and Corollary~\ref{cor:local-disparity}.
\end{itemize}

\subsection{Synthetic data-generating processes}
\label{app:synthetic-dgp}

\paragraph{Common construction.}
The synthetic processes start from a latent simplex mean $\mu(X)$. Writing $z=\mathrm{ilr}(\mu(X))$, the construction separates outcome noise from predictor bias:
\[
U=\mathrm{ilr}^{-1}\bigl(z+b(X,z)\bigr),\qquad
Y=\mathrm{ilr}^{-1}\bigl(z+\sigma(\mu(X))\varepsilon\bigr).
\]
Here $b=0$ outside the predictor-bias regime; $\varepsilon$ is Gaussian or heavy-tailed as specified by the configuration. This distinction matters for D4: changing predictor bias changes $U$, not the latent targets or noise. Scores are computed from the resulting $(Y,U)$ pairs.

\paragraph{Regime definitions.}
D1 uses constant residual scale as a negative control. D2 uses $\sigma(U)=\sigma_{\min}+c(1-H(U)/\log K)$, making low-entropy predictions harder. D3 uses step-function scales aligned with designated prediction-space groups. D4 adds a location-dependent displacement $b(X,z)$ to the predictor in ILR space; its angularly varying construction is specified in Appendix~\ref{app:additional-predictor}. D5 replaces the smooth-scale Gaussian noise with heavy-tailed $t$ noise. D6 increases the simplex dimension to $K=50$ and the calibration size; exact-reference runs use smaller samples as explicitly noted.

\paragraph{Synthetic sample sizes and distributions.}
Each repetition regenerates fresh synthetic data with separate train, calibration, scale-estimation, and test subsets determined by the YAML configuration. Low-dimensional regimes use benchmark-scale repetitions with large test sets so that stratified coverage is stable. D6 uses larger nominal calibration sizes to avoid trivial high-variance behavior from undersampling alone. Auxiliary exact-method runs use smaller calibration or test sizes only when the full benchmark configuration is computationally infeasible.

\subsection{Real-data benchmark distributions, preprocessing, and predictors}

\paragraph{Common principle.}
Each real task is converted into a fixed pair $(Y,U)$ before conformal evaluation. Here $Y$ is the target composition, $U$ is the predictor output on the simplex, and the conformal benchmark operates only on repeated calibration/test splits of these frozen arrays. This design keeps the paper focused on uncertainty wrappers rather than on repeated predictor retraining.

\paragraph{CIFAR-10 softmax.}
For CIFAR-10~\cite{krizhevsky2009learning}, the target distribution is one-hot over 10 classes and the predictor output is a cached 10-way softmax vector. The script first checks for cached predictions under \texttt{data/processed/} and downloads the raw torchvision data only if the cache is absent. The evaluation uses the full cached 10k-sample array for the all-method comparison. Entropy of the softmax vector provides the natural prediction-space stratification and reveals a pronounced hard-versus-easy split that drives the CIFAR result.

\paragraph{20 Newsgroups topics.}
For the topic benchmark, we use the 20 Newsgroups collection originating with Lang's Newsweeder corpus~\cite{lang1995newsweeder}. Both $Y$ and $U$ are 10-dimensional topic-mixture vectors. Preprocessing fixes the topic dimensionality once, constructs topic compositions, and caches the resulting arrays before conformal evaluation. The task is designed as a case of mild, smooth heterogeneity rather than a sharply discrete partition. In the archived run, normalization-based strategies help while grouped calibration has the highest disparity among non-weighted wrappers. In the stratification analyses, Mondrian ranks first under most maps, although the fixed-map series favors Global with entropy bins (Appendix~\ref{app:additional-maps}).

\paragraph{SemEval AffectiveText.}
For AffectiveText~\cite{strapparava2007semeval}, the raw archive is extracted under \texttt{data/raw/AffectiveText.Semeval.2007}. The loader searches permissively for XML files containing headlines and gold files containing six emotion intensities, then keeps only shared instance IDs. Gold scores are clipped away from zero and row-normalized to produce the target simplex composition. Prediction generation is handled in a separate cache-building script: each headline is sent to a fixed zero-shot API prompt requesting six scores in a fixed order, responses are parsed with a strict numeric extractor, zero-sum or malformed outputs are rejected, and the full response is stored in JSONL format. The frozen benchmark cache used in this paper was produced with an OpenAI-compatible client against BigModel's \texttt{glm-4-plus} endpoint. The benchmark itself only reads the frozen cache.

\paragraph{Samson hyperspectral unmixing.}
For Samson, we use the public 95$\times$95 ROI and ground-truth benchmark files summarized by Zhu~\cite{zhu2017spectral}. The target distribution is the abundance vector over endmembers and the predictor output is the abundance estimate from NMF-style unmixing. The benchmark uses the frozen unmixing output as the simplex predictor and varies only the conformal wrapper. Because the scene has a small number of endmembers and a physically meaningful abundance geometry, its heterogeneity is well aligned with grouped calibration.

\paragraph{PBMC pseudobulk deconvolution.}
PBMC is a semi-synthetic benchmark built from the PBMC3K single-cell reference~\cite{zheng2017massively}. The script converts sparse matrices to dense arrays, checks for an existing cell-type label, and, if necessary, creates a surrogate annotation by PCA followed by 8-cluster KMeans. Pseudobulk mixtures are then generated with Dirichlet-distributed proportions and additive noise, and NNLS deconvolution is run once to obtain a fixed predictor output $U$. Subsequent repetitions only resample calibration/test splits of the frozen residuals. This is why PBMC functions as a semi-synthetic control task rather than as a deconvolution-training benchmark.

\paragraph{UTKFace age label distributions.}
UTKFace~\cite{zhang2017age} uses an image-based predictor rather than the earlier degenerate target-derived construction. The retained paper version computes thumbnail-level image features, applies an image-space kNN regressor to predict age label distributions, and then evaluates conformal wrappers on the resulting simplex outputs. This revision is essential because it removes the near-deterministic coupling between target construction and prediction that made earlier versions of the task uninformative.

\paragraph{Predictor construction summary.}
Each real-data task fixes one predictor before conformal evaluation. CIFAR-10 uses a pretrained softmax classifier, Topics uses a topic-mixture predictor, Samson uses NMF-based unmixing, PBMC uses a deconvolution predictor, UTKFace uses an image-based age-distribution predictor, and AffectiveText uses a frozen zero-shot language-model scorer with cached outputs. The conformal benchmark then varies only the uncertainty wrapper.

\paragraph{AffectiveText prompt and cache contents.}
For AffectiveText, the predictor maps each headline to six emotion scores corresponding to anger, disgust, fear, joy, sadness, and surprise. The cached benchmark run was generated with an OpenAI-compatible client pointed at BigModel's \texttt{glm-4-plus} endpoint (\texttt{https://open.bigmodel.cn/api/paas/v4}) using a fixed prompt template that requests six comma-separated scores in that order. The internal cache used to create the reported arrays stores the headline ID, parsed scores, provider/base-URL metadata, prompt template, and raw API response for auditability. The releasable artifact omits raw headlines and raw responses; it ships only headline IDs, parsed scores, redacted provider metadata, the prompt template, and the schema needed to audit or locally rebuild the cache. This avoids dependence on live API calls during benchmarking and turns the task into a frozen evaluation artifact rather than a moving target. To reduce closed-model risk further, the artifact also ships a fully open fallback rebuild path: \texttt{scripts/cache\_affective\_text\_open\_predictions.py} fits a deterministic TF--IDF+SVD+kNN out-of-fold regressor on the public SemEval headlines and writes the same JSONL cache format used by the benchmark runner. This fallback is not the main predictor used for the reported tables, but it provides a reviewer-usable local reproduction path with no API access.

\paragraph{Data cleaning summary.}
Across all real tasks, the cleaning strategy is conservative: preserve the original task target whenever possible, convert it once into a simplex-valued representation, cache the resulting $(Y,U)$ arrays, and keep all conformal evaluation downstream of those frozen artifacts. The main cleaning operations are identifier matching and normalization for AffectiveText, cache reuse for CIFAR-10, feature-based predictor revision for UTKFace, and surrogate cell-type labeling plus pseudobulk generation for PBMC.

\subsection{Asset provenance, usage terms, and compute environment}

\begin{table*}[t]
\centering
\caption{Real-benchmark asset provenance and usage terms. Because the paper contributes an evaluation artifact, these source and usage conditions are part of the benchmark specification rather than ancillary metadata.}
\label{tab:asset-provenance}
\setlength{\tabcolsep}{3pt}
\begin{tabular}{@{}p{0.12\linewidth}p{0.25\linewidth}p{0.14\linewidth}p{0.41\linewidth}@{}}
\toprule
\textbf{Task} & \textbf{Source asset} & \textbf{Citation} & \textbf{Usage terms in this benchmark} \\
\midrule
CIFAR-10 & Official CIFAR-10 image benchmark, accessed through cached torchvision downloads & \cite{krizhevsky2009learning} & Public benchmark dataset used only to build frozen softmax predictions. The paper cites the official technical report requested by the source page and does not redistribute the image archive as part of the benchmark tables or figures. \\
Topics & 20 Newsgroups text collection, accessed through scikit-learn's fetcher & \cite{lang1995newsweeder} & Public educational/research benchmark used only to derive cached topic-mixture arrays. We treat the corpus as a source-cited benchmark asset rather than republishing the raw text collection. \\
Affective-\newline Text & SemEval-2007 Task 14 headlines and gold emotion labels & \cite{strapparava2007semeval} & Shared-task benchmark data used locally for evaluation. We do not repackage the raw headline collection in the paper outputs, and the conformal benchmark operates on derived cached predictions and normalized gold labels. \\
Samson & Public 95$\times$95 Samson ROI plus endmember/abundance benchmark files used in hyperspectral unmixing & \cite{zhu2017spectral} & Academic benchmark files used unmodified with attribution. The public bundle we used does not ship an explicit license file, so we document it as a source-cited research asset and avoid presenting it as a newly redistributable paper asset. \\
PBMC & 10x Genomics PBMC3K single-cell reference, accessed through \texttt{scanpy.datasets.}\allowbreak\texttt{pbmc3k} & \cite{zheng2017massively} & Public 10x tutorial data. Our benchmark uses derived pseudobulk mixtures and frozen deconvolution outputs rather than redistributing the raw count matrix. Source terms and any permitted redistribution are documented separately in the task card. \\
UTKFace & Official UTKFace aligned-and-cropped face images & \cite{zhang2017age} & Non-commercial research only according to the dataset homepage. Copyright remains with the original image owners, so the benchmark uses derived features and predictions and does not repackage the face-image archive in the paper outputs. \\
\bottomrule
\end{tabular}
\end{table*}

\subsection{Omitted cells and computational constraints}

\paragraph{Why some table entries are marked \texttt{--}.}
The missing entries in the appendix summary tables do not indicate hidden failures. They reflect either methods that were not run in a given benchmark slice or quantities that are not meaningful for that task. Oracle rows are only defined when a meaningful residual-scale oracle is available, train-residual normalization is not well-defined for fixed deconvolution pipelines such as PBMC where no separate predictor-training residuals exist, and some exact methods were omitted in the largest high-dimensional settings because their computational cost scales poorly with the number of calibration and test points. The dagger attached to D6 in the main text denotes this exact issue: the D6$^{\dagger}$ FullCP and Jackknife+ cells are drawn from a smaller exact-reference run rather than from the benchmark-scale D6 configuration used for the split methods.

\paragraph{Weighted CP status.}
Weighted CP is included for completeness because it is an important family in the literature, but the present implementation uses inverse local-scale weights rather than covariate-shift density-ratio weights, and this inverse-scale rule is unstable on the simplex. We therefore interpret weighted results in this paper as diagnostic failure cases rather than as the final word on weighting-based calibration.

\subsection{Runtime notes}

\paragraph{Representative runtimes.}
FullCP costs roughly 2.30 seconds per repetition on D2, 3.35 on Topics, 3.71 on Samson, 3.25 on full-scale CIFAR-10, 0.70 on the 2k-reference UTKFace run, and 0.31 on AffectiveText. These are recorded runtimes for the respective configurations, not a hardware-independent complexity comparison. Split, partition, and normalization-based methods are typically cheaper.

\begin{table}[t]
\centering
\caption{Compute environment and repetition counts for the benchmark.}
\label{tab:compute}
\setlength{\tabcolsep}{3pt}
\begin{tabular}{@{}p{0.32\linewidth}p{0.6\linewidth}@{}}
\toprule
\textbf{Item} & \textbf{Benchmark setting} \\
\midrule
CPU-only conformal evaluation on a personally owned Apple MacBook Air. \\
Predictor regeneration & Not part of repeated benchmarking. Each task first constructs or caches a fixed simplex predictor $U$, after which all conformal runs operate on frozen arrays. \\
Synthetic repetitions & 200 random calibration/test splits per regime. \\
Real repetitions & 50 random calibration/test splits for most tasks; AffectiveText and PBMC use 200. \\
Cheap methods & Global, Mondrian/partition, and most normalization-style methods are typically sub-second per repetition once $(Y,U)$ is cached. \\
Expensive methods & FullCP: 2.30 sec/rep on D2, 3.35 on Topics, 3.71 on Samson, and 3.25 on full-scale CIFAR-10. \\
\bottomrule
\end{tabular}
\end{table}

\paragraph{High-dimensional exact methods.}
The D6 high-$K$ setting exposes a practical limitation of exact conformal strategies. At the benchmark-scale configuration used in the main synthetic suite, exact and leave-one-out style methods were too slow to include in the main comparison table on the same footing as the split methods. Throughout this appendix, D6$^{\dagger}$ means that the FullCP and Jackknife+ entries come from a smaller exact-reference run with $n_{\mathrm{train}}=250$, $n_{\mathrm{cal}}=250$, $n_{\mathrm{test}}=1000$, and 50 repetitions, whereas the other D6 cells use the main high-$K$ benchmark scale. We keep those exact-reference entries because they clarify the qualitative placement of exact methods in the high-dimensional regime, but they should be read as computational reference points rather than as directly comparable leaderboard numbers.

\subsection{Representative stratified profiles and compute tradeoffs}

\begin{figure*}[t]
\centering
\includegraphics[width=0.95\textwidth]{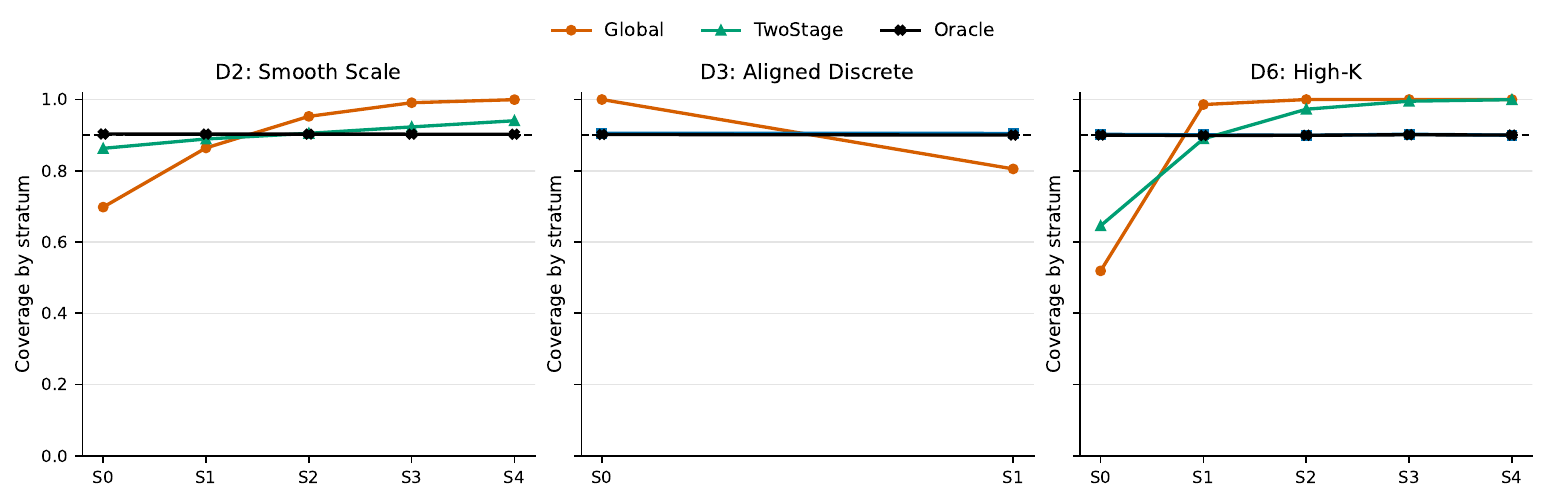}
\caption{Representative synthetic stratified coverage profiles. Normalization is competitive in D2, aligned grouped calibration helps in D3, and global coverage is uneven in D6. These profiles do not establish a unique winner among wrapper families; paired synthetic comparisons are reported separately.}
\label{fig:appendix-synth-profiles}
\end{figure*}

\begin{figure*}[t]
\centering
\includegraphics[width=0.98\textwidth]{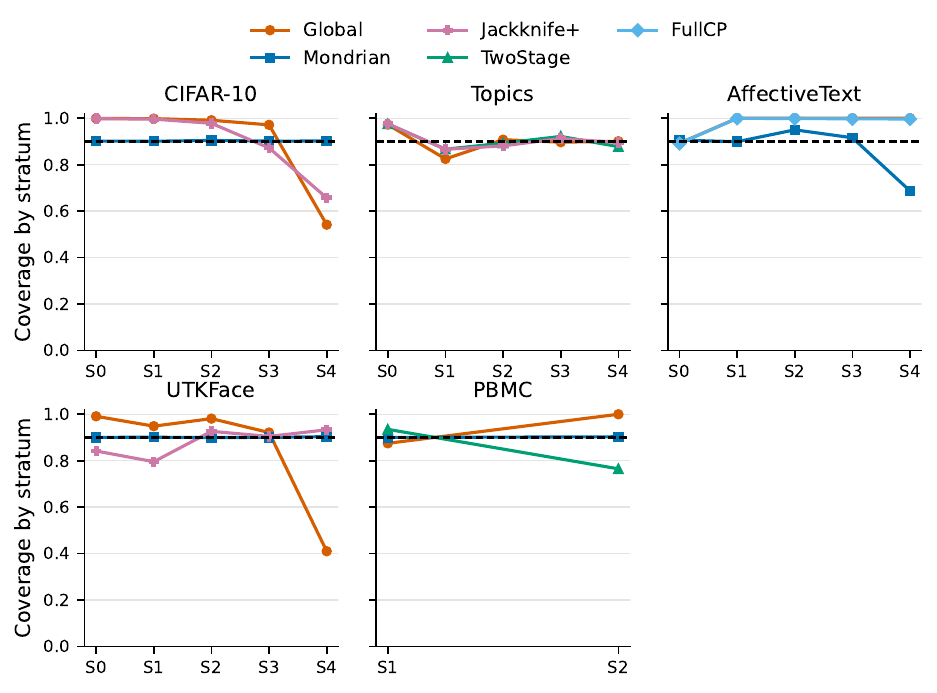}
\caption{Stratified coverage from the archived task runs used in the summary tables. Grouped repair helps on CIFAR-10 and UTKFace; Topics and AffectiveText have different rankings under their archived protocols. PBMC displays only methods from its two-stratum main protocol, not the incompatible five-stratum auxiliary runs.}
\label{fig:appendix-real-profiles}
\end{figure*}

\begin{figure}[t]
\centering
\includegraphics[width=0.75\linewidth]{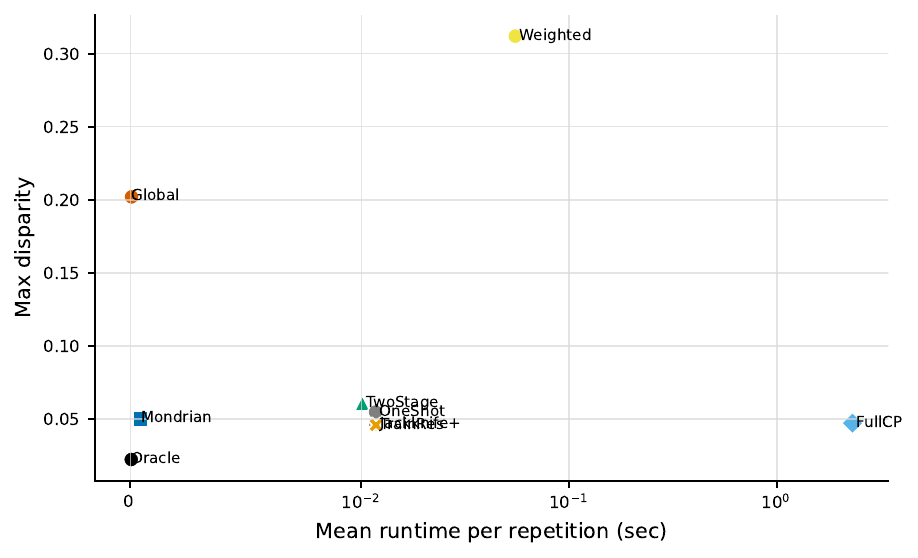}
\caption{Runtime-versus-disparity tradeoff on the smooth-scale synthetic benchmark. Exact and leave-one-out variants can improve allocation reliability, but they do so at a nontrivial computational cost relative to split, partition, and normalization-based methods.}
\label{fig:appendix-runtime}
\end{figure}

\subsection{Real-task stratification sensitivity}

\paragraph{Why this appendix is necessary.}
The main text uses one pre-declared default stratification per task, but the benchmark's core allocation claims should not rest on a single binning choice. We therefore rerun every real task under four fixed alternative stratifications that are declared before evaluation: boundary bins, entropy bins, dominant-component or dominant-group strata, and KMeans strata on the cached prediction vectors. Table~\ref{tab:real-strata-sensitivity} reports the best valid wrapper under each stratification. The resulting picture is more stable than the default-view story alone suggests: Mondrian is the best valid method on five of the six real tasks under all four fixed stratifications, and AffectiveText is the lone mixed case, flipping to Global under dominant-component bins. This robustness analysis uses the earlier task-specific runners of the archived series. The unified fixed-map sweep in Appendix~\ref{app:additional-maps} is a separate two-method experiment and is not a numerical reproduction of this table; the two differ on Topics with entropy bins. The purpose is not to claim stratification invariance, but to show which regime-level conclusions are stable, which ones are sensitive, and which tasks should be presented as ambiguity-bearing counterexamples rather than clean positives.

\input{generated_real_strata_sensitivity}

\paragraph{PBMC depends on generation and stratification.}
Mondrian remains the best empirically screened wrapper in the fixed-map sweep. The archived mixture-generation runs show a different pattern: it performs worse than Global for concentration 0.5 and similarly for concentration 2.0. The concentration-1.0 point uses the retained main run, with two populated strata. Although these runs use the boundary-based stratification family, they do not hold an identical realized partition fixed. The figure therefore documents sensitivity across configurations, not a controlled estimate of the effect of concentration alone.

\begin{figure*}[t]
\centering
\includegraphics[width=0.9\textwidth]{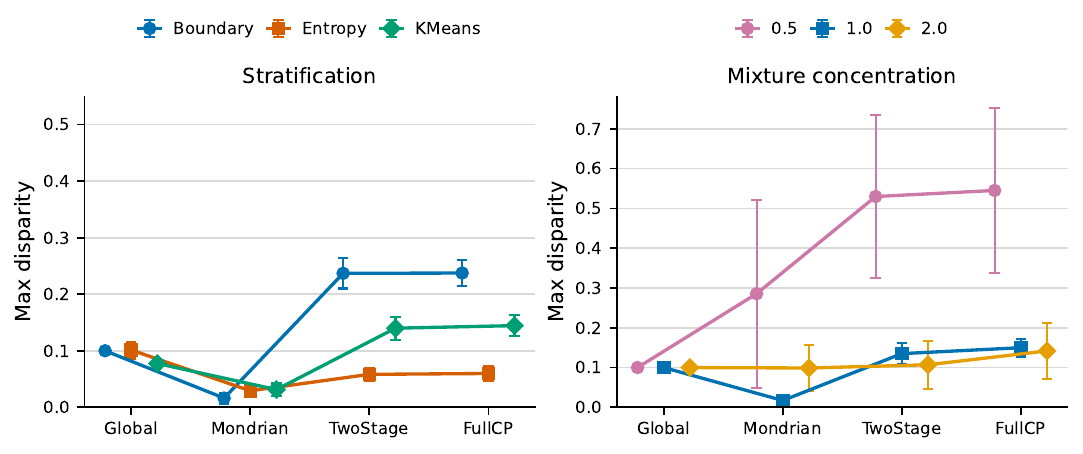}
\caption{PBMC sensitivity checks. Left: alternative fixed stratifications of frozen outputs. Right: archived pseudobulk configurations with Dirichlet concentrations 0.5, 1.0, and 2.0; the 1.0 point uses the retained two-stratum main run. The right panel changes generated data and realized strata as well as concentration, so it is descriptive rather than a concentration-only causal comparison.}
\label{fig:appendix-pbmc-sensitivity}
\end{figure*}

\input{generated_appendix_tables}

\clearpage
\input{additional_evaluation}

\clearpage

\end{document}

%% file: generated_real_strata_sensitivity.tex
\begin{table*}[t]
\centering
\caption{Real-task stratification sensitivity across fixed alternative strata. Each cell reports the best valid method under that stratification, shown as method name with $(\text{max disparity}, \text{marginal coverage})$. A method is treated as valid when mean marginal coverage is at least nominal minus the configured tolerance.}
\label{tab:real-strata-sensitivity}
\scriptsize
\resizebox{\textwidth}{!}{%
\begin{tabular}{@{}lccccp{3.3cm}@{}}
\toprule
Task & Boundary & Entropy & Dominant & KMeans & Winner stability \\
\midrule
CIFAR-10 & Mondrian (0.021, 0.901) & Mondrian (0.022, 0.902) & Mondrian (0.020, 0.902) & Mondrian (0.024, 0.902) & Stable: Mondrian \\
Topics & Mondrian (0.024, 0.902) & Mondrian (0.025, 0.902) & Mondrian (0.038, 0.901) & Mondrian (0.028, 0.902) & Stable: Mondrian \\
AffectiveText & Mondrian (0.070, 0.919) & Mondrian (0.071, 0.919) & Global (0.106, 0.906) & Mondrian (0.073, 0.918) & Mixed: Global, Mondrian \\
Samson & Mondrian (0.020, 0.902) & Mondrian (0.025, 0.903) & Mondrian (0.011, 0.901) & Mondrian (0.079, 0.903) & Stable: Mondrian \\
UTKFace & Mondrian (0.014, 0.900) & Mondrian (0.013, 0.901) & Mondrian (0.027, 0.901) & Mondrian (0.019, 0.902) & Stable: Mondrian \\
PBMC & Mondrian (0.016, 0.902) & Mondrian (0.029, 0.903) & Mondrian (0.031, 0.903) & Mondrian (0.031, 0.904) & Stable: Mondrian \\
\bottomrule
\end{tabular}
}
\end{table*}

%% file: generated_appendix_tables.tex
\FloatBarrier\subsection{Real-task summary tables}

\begin{table*}[!htbp]\centering
\scriptsize
\caption{Real-task metrics from the same saved sources as the original figure. $\dagger$ denotes reference-only cells: UTKFace FullCP uses 2,000 samples; PBMC auxiliary methods use a different five-stratum evaluation from the two-stratum main run. These cells do not establish head-to-head rankings. Weighted is an inverse-scale diagnostic. Runtime is taken from each method's own run, not another configuration.}\label{tab:real-summary-1}
\setlength{\tabcolsep}{4pt}
\begin{tabular}{@{}lllllll@{}}\toprule
Task & Method & Coverage & Worst & Disparity & Radius & Runtime (s) \\\midrule
CIFAR-10 & Global & 0.900 & 0.542 & 0.358 & 0.153 & 0.000 \\
CIFAR-10 & Mondrian & 0.902 & 0.886 & 0.022 & 0.188 & 0.001 \\
CIFAR-10 & TwoStage & 0.901 & 0.641 & 0.259 & 0.119 & 0.079 \\
CIFAR-10 & FullCP & 0.900 & 0.656 & 0.244 & 0.107 & 3.247 \\
CIFAR-10 & Jackknife+ & 0.900 & 0.657 & 0.243 & 0.108 & 0.151 \\
CIFAR-10 & OneShot & 0.882 & 0.609 & 0.291 & 0.091 & 0.150 \\
CIFAR-10 & TrainRes & 0.901 & 0.663 & 0.237 & 0.106 & 0.150 \\
CIFAR-10 & \makecell[l]{Weighted\\(inv-scale)} & 0.130 & 0.000 & 0.900 & 0.000 & 0.401 \\
Samson & Global & 0.902 & 0.762 & 0.138 & 27.146 & 0.000 \\
Samson & Mondrian & 0.903 & 0.884 & 0.025 & 28.284 & 0.001 \\
Samson & TwoStage & 0.900 & 0.800 & 0.101 & 19.716 & 0.020 \\
Samson & FullCP & 0.900 & 0.805 & 0.096 & 19.257 & 3.712 \\
Samson & Jackknife+ & 0.901 & 0.805 & 0.095 & 19.271 & 0.027 \\
Samson & OneShot & 0.890 & 0.783 & 0.117 & 19.000 & 0.027 \\
Samson & TrainRes & 0.901 & 0.803 & 0.097 & 19.163 & 0.027 \\
Samson & \makecell[l]{Weighted\\(inv-scale)} & 0.622 & 0.510 & 0.390 & 26.664 & 0.296 \\
Topics & Global & 0.899 & 0.820 & 0.098 & 6.025 & 0.000 \\
Topics & Mondrian & 0.900 & 0.654 & 0.276 & 6.042 & 0.000 \\
Topics & TwoStage & 0.899 & 0.854 & 0.092 & 5.842 & 0.070 \\
Topics & FullCP & 0.899 & 0.857 & 0.091 & 5.813 & 3.347 \\
Topics & Jackknife+ & 0.900 & 0.857 & 0.091 & 5.817 & 0.133 \\
Topics & OneShot & 0.888 & 0.847 & 0.092 & 5.735 & 0.132 \\
Topics & TrainRes & 0.901 & 0.861 & 0.093 & 5.791 & 0.131 \\
Topics & \makecell[l]{Weighted\\(inv-scale)} & 0.189 & 0.165 & 0.735 & 3.611 & 0.211 \\
\bottomrule\end{tabular}\end{table*}

\begin{table*}[!htbp]\centering
\scriptsize
\caption{Real-task metrics from the same saved sources as the original figure. $\dagger$ denotes reference-only cells: UTKFace FullCP uses 2,000 samples; PBMC auxiliary methods use a different five-stratum evaluation from the two-stratum main run. These cells do not establish head-to-head rankings. Weighted is an inverse-scale diagnostic. Runtime is taken from each method's own run, not another configuration.}\label{tab:real-summary-2}
\setlength{\tabcolsep}{4pt}
\begin{tabular}{@{}lllllll@{}}\toprule
Task & Method & Coverage & Worst & Disparity & Radius & Runtime (s) \\\midrule
AffectiveText & Global & 0.904 & 0.892 & 0.100 & 29.073 & 0.000 \\
AffectiveText & Mondrian & 0.902 & 0.582 & 0.356 & 28.893 & 0.000 \\
AffectiveText & TwoStage & 0.905 & 0.893 & 0.100 & 86.693 & 0.003 \\
AffectiveText & FullCP & 0.903 & 0.889 & 0.101 & 38.979 & 0.308 \\
AffectiveText & Jackknife+ & 0.894 & 0.880 & 0.101 & 37.746 & 0.004 \\
AffectiveText & OneShot & 0.890 & 0.873 & 0.102 & 34.973 & 0.004 \\
AffectiveText & TrainRes & 0.903 & 0.891 & 0.101 & 44.619 & 0.004 \\
AffectiveText & \makecell[l]{Weighted\\(inv-scale)} & 0.770 & 0.753 & 0.152 & 24.811 & 0.006 \\
UTKFace & Global & 0.900 & 0.410 & 0.490 & 3.164 & 0.000 \\
UTKFace & Mondrian & 0.901 & 0.885 & 0.021 & 1.892 & 0.001 \\
UTKFace & TwoStage & 0.901 & 0.790 & 0.116 & 4.713 & 0.075 \\
UTKFace & FullCP$^{\dagger}$ & 0.902 & 0.764 & 0.141 & 2.904 & 0.703 \\
UTKFace & Jackknife+ & 0.900 & 0.795 & 0.105 & 1.524 & 0.114 \\
UTKFace & OneShot & 0.887 & 0.778 & 0.122 & 1.379 & 0.116 \\
UTKFace & TrainRes & 0.900 & 0.793 & 0.107 & 1.478 & 0.117 \\
UTKFace & \makecell[l]{Weighted\\(inv-scale)} & 0.335 & 0.000 & 0.900 & 0.000 & 2.404 \\
PBMC & Global & 0.900 & 0.875 & 0.100 & 42.265 & 0.000 \\
PBMC & Mondrian & 0.901 & 0.893 & 0.017 & 39.982 & 0.000 \\
PBMC & TwoStage & 0.901 & 0.765 & 0.135 & 38.007 & 0.018 \\
PBMC & FullCP & 0.900 & 0.750 & 0.150 & 37.134 & 1.055 \\
PBMC & Jackknife+$^{\dagger}$ & 0.900 & 0.526 & 0.374 & 37.216 & 0.047 \\
PBMC & OneShot$^{\dagger}$ & 0.892 & 0.525 & 0.375 & 36.744 & 0.046 \\
PBMC & TrainRes$^{\dagger}$ & 0.900 & 0.535 & 0.365 & 37.019 & 0.028 \\
PBMC & \makecell[l]{Weighted\\(inv-scale)}$^{\dagger}$ & 0.638 & 0.588 & 0.312 & 30.147 & 0.142 \\
\bottomrule\end{tabular}\end{table*}

\clearpage\FloatBarrier\subsection{Per-strata coverage tables}

\begin{table*}[!htbp]\centering
\small
\caption{CIFAR-10: mean stratum coverage from the corresponding summary sources. Reference cells ($\dagger$) are not paired with main-run cells. For PBMC the two maps have different group meanings and are tabulated separately.}\label{tab:strata-cifar-10}
\setlength{\tabcolsep}{4pt}
\begin{tabular}{@{}lllllllll@{}}\toprule
Stratum & Global & Mondrian & TwoStage & FullCP & Jackknife+ & OneShot & TrainRes & Weighted \\\midrule
S0 & 0.999 & 0.902 & 0.999 & 0.998 & 0.998 & 0.998 & 0.999 & 0.650 \\
S1 & 0.998 & 0.901 & 0.998 & 0.996 & 0.996 & 0.995 & 0.996 & 0.000 \\
S2 & 0.992 & 0.905 & 0.983 & 0.978 & 0.978 & 0.970 & 0.977 & 0.000 \\
S3 & 0.972 & 0.902 & 0.882 & 0.872 & 0.872 & 0.838 & 0.871 & 0.000 \\
S4 & 0.542 & 0.902 & 0.641 & 0.656 & 0.657 & 0.609 & 0.663 & 0.000 \\
\bottomrule\end{tabular}\end{table*}

\begin{table*}[!htbp]\centering
\small
\caption{Samson: mean stratum coverage from the corresponding summary sources. Reference cells ($\dagger$) are not paired with main-run cells. For PBMC the two maps have different group meanings and are tabulated separately.}\label{tab:strata-samson}
\setlength{\tabcolsep}{4pt}
\begin{tabular}{@{}lllllllll@{}}\toprule
Stratum & Global & Mondrian & TwoStage & FullCP & Jackknife+ & OneShot & TrainRes & Weighted \\\midrule
S0 & 0.762 & 0.902 & 0.953 & 0.956 & 0.956 & 0.953 & 0.955 & 0.649 \\
S1 & 0.972 & 0.902 & 0.883 & 0.885 & 0.885 & 0.878 & 0.888 & 0.529 \\
S2 & 0.958 & 0.902 & 0.895 & 0.900 & 0.901 & 0.884 & 0.901 & 0.510 \\
S3 & 0.934 & 0.904 & 0.893 & 0.877 & 0.878 & 0.864 & 0.879 & 0.673 \\
S4 & 1.000 & 0.904 & 0.800 & 0.805 & 0.805 & 0.783 & 0.803 & 0.873 \\
\bottomrule\end{tabular}\end{table*}

\begin{table*}[!htbp]\centering
\small
\caption{Topics: mean stratum coverage from the corresponding summary sources. Reference cells ($\dagger$) are not paired with main-run cells. For PBMC the two maps have different group meanings and are tabulated separately.}\label{tab:strata-topics}
\setlength{\tabcolsep}{4pt}
\begin{tabular}{@{}lllllllll@{}}\toprule
Stratum & Global & Mondrian & TwoStage & FullCP & Jackknife+ & OneShot & TrainRes & Weighted \\\midrule
S0 & 0.973 & 0.729 & 0.976 & 0.976 & 0.976 & 0.976 & 0.977 & 0.887 \\
S1 & 0.825 & 0.913 & 0.867 & 0.866 & 0.866 & 0.856 & 0.871 & 0.274 \\
S2 & 0.908 & 0.903 & 0.893 & 0.881 & 0.881 & 0.869 & 0.885 & 0.212 \\
S3 & 0.897 & 0.899 & 0.922 & 0.912 & 0.913 & 0.901 & 0.912 & 0.165 \\
S4 & 0.901 & 0.900 & 0.878 & 0.895 & 0.895 & 0.883 & 0.896 & 0.200 \\
\bottomrule\end{tabular}\end{table*}

\begin{table*}[!htbp]\centering
\small
\caption{AffectiveText: mean stratum coverage from the corresponding summary sources. Reference cells ($\dagger$) are not paired with main-run cells. For PBMC the two maps have different group meanings and are tabulated separately.}\label{tab:strata-affectivetext}
\setlength{\tabcolsep}{4pt}
\begin{tabular}{@{}lllllllll@{}}\toprule
Stratum & Global & Mondrian & TwoStage & FullCP & Jackknife+ & OneShot & TrainRes & Weighted \\\midrule
S0 & 0.892 & 0.906 & 0.894 & 0.892 & 0.882 & 0.877 & 0.892 & 0.753 \\
S1 & 1.000 & 0.899 & 0.996 & 1.000 & 0.999 & 0.999 & 0.999 & 0.903 \\
S2 & 1.000 & 0.950 & 0.995 & 0.999 & 0.997 & 0.996 & 0.997 & 0.903 \\
S3 & 1.000 & 0.916 & 0.997 & 0.998 & 0.997 & 0.995 & 0.998 & 0.936 \\
S4 & 1.000 & 0.687 & 0.998 & 0.997 & 0.997 & 0.994 & 1.000 & 0.973 \\
\bottomrule\end{tabular}\end{table*}

\begin{table*}[!htbp]\centering
\small
\caption{UTKFace: mean stratum coverage from the corresponding summary sources. Reference cells ($\dagger$) are not paired with main-run cells. For PBMC the two maps have different group meanings and are tabulated separately.}\label{tab:strata-utkface}
\setlength{\tabcolsep}{4pt}
\begin{tabular}{@{}lllllllll@{}}\toprule
Stratum & Global & Mondrian & TwoStage & FullCP$^{\dagger}$ & Jackknife+ & OneShot & TrainRes & Weighted \\\midrule
S0 & 0.991 & 0.901 & 0.811 & 0.988 & 0.842 & 0.831 & 0.844 & 0.456 \\
S1 & 0.949 & 0.901 & 0.800 & 0.973 & 0.796 & 0.778 & 0.793 & 0.378 \\
S2 & 0.981 & 0.900 & 0.917 & 0.987 & 0.926 & 0.921 & 0.927 & 0.397 \\
S3 & 0.921 & 0.901 & 0.905 & 0.861 & 0.905 & 0.891 & 0.905 & 0.327 \\
S4 & 0.410 & 0.904 & 0.999 & 0.764 & 0.933 & 0.902 & 0.931 & 0.000 \\
\bottomrule\end{tabular}\end{table*}

\begin{table*}[!htbp]\centering
\small
\caption{PBMC: mean stratum coverage from the corresponding summary sources. main run. Reference cells ($\dagger$) are not paired with main-run cells. For PBMC the two maps have different group meanings and are tabulated separately.}\label{tab:strata-pbmc}
\setlength{\tabcolsep}{4pt}
\begin{tabular}{@{}lllll@{}}\toprule
Stratum & Global & Mondrian & TwoStage & FullCP \\\midrule
S1 & 0.875 & 0.900 & 0.935 & 0.937 \\
S2 & 1.000 & 0.903 & 0.765 & 0.750 \\
\bottomrule\end{tabular}\end{table*}

\begin{table*}[!htbp]\centering
\small
\caption{PBMC: mean stratum coverage from the corresponding summary sources. reference run. Reference cells ($\dagger$) are not paired with main-run cells. For PBMC the two maps have different group meanings and are tabulated separately.}\label{tab:strata-pbmc-reference}
\setlength{\tabcolsep}{4pt}
\begin{tabular}{@{}lllll@{}}\toprule
Stratum & Jackknife+$^{\dagger}$ & OneShot$^{\dagger}$ & TrainRes$^{\dagger}$ & Weighted$^{\dagger}$ \\\midrule
S0 & 0.919 & 0.911 & 0.917 & 0.588 \\
S1 & 0.759 & 0.753 & 0.764 & 0.970 \\
S2 & 0.802 & 0.795 & 0.808 & 0.987 \\
S3 & 0.833 & 0.826 & 0.836 & 0.974 \\
S4 & 0.530 & 0.528 & 0.539 & 0.993 \\
\bottomrule\end{tabular}\end{table*}

\clearpage\FloatBarrier\subsection{Runtime comparison}

\begin{table*}[!htbp]\centering
\small
\caption{Seconds per repetition from each cell's own result file. UTKFace FullCP and PBMC auxiliary timings retain the reference-run limitations of Tables~\ref{tab:real-summary-1}--\ref{tab:real-summary-2}; they are not normalized speedup comparisons.}\label{tab:runtime-real}
\setlength{\tabcolsep}{4pt}
\begin{tabular}{@{}lllllll@{}}\toprule
Method & CIFAR-10 & Samson & Topics & AffectiveText & UTKFace & PBMC \\\midrule
Global & 0.0001 & 0.0001 & 0.0001 & 0.0001 & 0.0002 & 0.0000 \\
Mondrian & 0.0005 & 0.0006 & 0.0005 & 0.0002 & 0.0012 & 0.0001 \\
TwoStage & 0.0794 & 0.0204 & 0.0703 & 0.0027 & 0.0752 & 0.0181 \\
FullCP & 3.2473 & 3.7123 & 3.3472 & 0.3085 & 0.7031 & 1.0547 \\
Jackknife+ & 0.1507 & 0.0274 & 0.1325 & 0.0036 & 0.1135 & 0.0471 \\
OneShot & 0.1495 & 0.0270 & 0.1317 & 0.0036 & 0.1161 & 0.0462 \\
TrainRes & 0.1501 & 0.0272 & 0.1309 & 0.0036 & 0.1172 & 0.0281 \\
Weighted & 0.4013 & 0.2958 & 0.2108 & 0.0059 & 2.4042 & 0.1421 \\
\bottomrule\end{tabular}\end{table*}

%% file: additional_evaluation.tex
\section{Supplementary Comparisons and Robustness Experiments}
\label{app:additional-evaluation}

\subsection{Terminology and interpretation}
An exchangeable joint distribution is unchanged by permuting its observations. Marginal validity averages coverage over the calibration sample and a new test point under the stated assumptions; prediction-space allocation describes how empirical coverage differs across subsets defined by $U$. A stratum is one cell of a partition $g(U)$. A local scale is a positive function used to normalize a residual before calibration. Mondrian calibrates separately within one partition; multivalid methods target a supplied family of groups that may overlap.

Aitchison distance is Euclidean distance in orthonormal ILR coordinates. Total variation is $\tfrac12\sum_k|y_k-u_k|$, equal to $1-u_j$ for one-hot target $y=e_j$. The latter is defined on the simplex boundary, unlike unregularized log-ratio distance. A radius is a threshold in the task's score units, not a volume. Larger radii give nested sets at a fixed center, but equal radii at different centers or under different scores need not imply equal volume. Neither absolute radii nor their rankings should be compared across tasks.

\input{audited_radius_ratios}
\FloatBarrier

\paragraph{Scope of the archived comparisons.}
The main heatmap retains the task-specific sources rather than substituting a different fixed-map experiment for each task. CIFAR-10 uses the explicitly fixed entropy run. Other historical configuration files do not uniformly record a fixed-map flag, so their rankings are descriptive, not proof of a protocol-independent advantage. PBMC's main Global, Mondrian, TwoStage, and FullCP records have two populated strata; its auxiliary Jackknife+, OneShot, TrainRes, and Weighted records have five. Their subgroup metrics are therefore not directly comparable. UTKFace FullCP is a separate 2,000-sample reference. The summary tables mark these cells, and PBMC stratum profiles are separated by protocol. This distinction is also essential when interpreting radius and runtime.

\subsection{Experiment series and provenance}
The original benchmark figures and tables and the supplementary experiments below are distinct series. The four-map sweep, count sweep, and real-task BatchMVP comparison use released frozen arrays with their own configurations and splits. They are robustness analyses, not numerical reproductions of every original table cell. Cached predictors, preprocessing, calibration fractions, and repetition counts can differ between series. In particular, agreement on a map name such as ``entropy'' is not sufficient to establish identical groups or samples. Pairing is claimed only within a specified experiment; cross-series comparisons are descriptive.

Both sensitivity sweeps use 40\% calibration and 60\% test observations, with 50 repetitions for most real tasks and 200 for AffectiveText and PBMC. Maps are computed from the complete response-free prediction array before splitting and remain fixed over repetitions. The BatchMVP experiment uses 200 synthetic repetitions and the same real-task repetition counts as these sweeps. Its synthetic draws are separate from the original all-wrapper runs. JSON result snapshots and their SHA256 manifest accompany this supplement.

\subsection{Paired synthetic comparisons}
\label{app:additional-synthetic-ci}
TwoStage is the primary normalization wrapper; OneShot and TrainRes diagnose alternative scale-fitting choices. FullCP and Jackknife+ are local-scale and leave-one-out references, not substitutes for comparing TwoStage directly. The following percentile bootstrap intervals use paired repetition differences from the saved synthetic runs. Auxiliary runs share the corresponding seeds and configurations; smaller D6 exact-reference runs are excluded from these main-scale paired comparisons. Intervals quantify repetition variability and are not multiplicity-adjusted.

In D2, normalization improves on Global, but Mondrian has lower disparity than TwoStage and its difference from FullCP is not resolved by the interval. In D5, the TwoStage--Mondrian and TrainRes--Jackknife+ comparisons also include zero. Smooth heterogeneity therefore supports several competitive repairs, not a uniquely best wrapper family.
\input{generated_synthetic_ci}
\FloatBarrier

\subsection{Group-wise and multivalid evaluation}
\label{app:additional-multivalid}
We evaluate an adaptation of BatchMVP~\cite{jung2023batch} with Global and Mondrian on all 12 tasks, sharing splits within each task. Candidate groups include the universal set, default strata, entropy bins, boundary bins, and grouped dominant components. Duplicate groups and groups with less than 5\% support in the response-free pool are removed, so the retained family need not contain every small default stratum. All three methods are evaluated on identical test observations and retained groups.

We report default-partition disparity $D_g$ and overlapping-family disparity $D_{\mathcal H}$. The latter evaluates coverage over groups, not coverage jointly conditional on group membership and threshold value, and is therefore not a complete test of multivalidity. The implementation transforms scores by $h(r)=r/(1+r)$, uses 1,000 midpoint thresholds and at most 1,000 patches, and initializes at the lowest threshold. Each patch targets the group/threshold cell with largest count-weighted squared coverage deviation, computes its empirical target quantile, and rounds to the grid. Patching stops when the grid index does not change or the cap is reached. Quantile interpolation, discretization, stopping, and group support can affect empirical behavior; we do not claim a new finite-sample guarantee for this implementation.

BatchMVP lowers mean $D_g$ against Global on 10/12 tasks and mean $D_{\mathcal H}$ on 11/12. Against Mondrian it lowers mean $D_{\mathcal H}$ on D2, D6, Topics, PBMC, and UTKFace, while Mondrian has lower $D_g$ on all 12 tasks. D2's overlap difference is small and its interval includes zero. BatchMVP mean coverage ranges from 0.885 to 0.962. All tasks pass the one-sided 0.88 screen, but CIFAR-10 (0.953) and Samson (0.962) overcover relative to 0.90. These results distinguish the two evaluation objectives without establishing universal superiority of either method.
\input{generated_batch_full}
\FloatBarrier

\subsection{Fixed-map definition and count sweeps}
\label{app:additional-maps}
Table~\ref{tab:additional-map-sweep} compares Global and Mondrian under boundary, entropy, grouped dominant-component, and ILR KMeans maps. Unlike Table~\ref{tab:real-strata-sensitivity}, it does not select the best valid method from the entire wrapper suite. Mondrian has lower mean disparity in 23/24 cells, with 21 paired intervals excluding zero in its favor. Topics--entropy favors Global; Topics--KMeans and AffectiveText--dominant intervals include zero. This is a separate released-array experiment, so its exception need not match the original all-method table. Intervals use 10,000 paired bootstrap resamples.
\input{additional_map_sweep}
\FloatBarrier

The count sweep requests 2, 3, 5, or 10 groups under the same fixed-map protocol. CIFAR-10, Samson, and UTKFace favor Mondrian throughout; Topics favors Global under entropy bins. AffectiveText changes ordering with partition granularity. Ties reduce realized PBMC group counts to one at requested counts 2/3, and two/three at requested counts 5/10. We report realized counts and minimum calibration support rather than assuming that every requested bin exists.
\input{generated_strata_count}
\FloatBarrier

\subsection{Controlled predictor-quality experiment}
\label{app:additional-predictor}
We vary D4 predictor bias while keeping the latent data-generating mechanism fixed. For $z=\operatorname{ilr}(\mu(X))\in\mathbb R^2$, let $r=\|z\|_2$ and $\theta=\operatorname{atan2}(z_2,z_1+10^{-8})$. Define
\[
b_\beta(z)=\beta(0.5+\tanh r)\begin{pmatrix}\cos(2\theta)\\\sin(2\theta)\end{pmatrix},\qquad
U_\beta=\operatorname{ilr}^{-1}(z+b_\beta(z)).
\]
Targets remain $Y=\operatorname{ilr}^{-1}(z+\sigma(\mu)\varepsilon)$. This is an angularly varying displacement, not a rigid rotation. We use $K=3$, $\sigma_{\min}=0.1$, $c=0.15$, and $\beta\in\{0.45,0.30,0.15,0\}$. The 200 repetitions share latent draws, noise, and calibration/test indices across bias levels, with 500 calibration and 5,000 test points. Boundary maps are recomputed on the complete $U_\beta$ array at each level, then fixed across calibration and test. Thus samples are paired, but the population partition is not held unchanged.

Removing bias reduces mean Aitchison score from 0.552 to 0.163. Global disparity falls from 0.151 to 0.128, a paired change of $-0.0236\,[-0.0312,-0.0165]$; worst-stratum coverage rises from 0.750 to 0.773, a change of $+0.0230\,[0.0154,0.0307]$. Mondrian disparity changes from 0.051 to 0.050, with difference $-0.0008\,[-0.0034,0.0018]$. Intervals are 95\% paired percentile bootstrap intervals with 10,000 resamples. Residual scale heterogeneity remains, and allocation is not monotone at every intermediate bias level. This is controlled evidence for investigating predictor quality, not a guarantee about retraining real-world predictors.
\input{generated_predictor_sweep}
\FloatBarrier

%% file: audited_radius_ratios.tex
\begin{table*}[!htbp]\centering
\small
\caption{Within-task ratios of mean radii, Mondrian divided by Global, computed from the same saved runs as the corrected summary tables. Ratios describe radial extent, not volume or clinical utility.}\label{tab:radius-ratios}
\setlength{\tabcolsep}{4pt}
\begin{tabular}{@{}llll@{}}\toprule
Task & Global & Mondrian & Ratio \\\midrule
CIFAR-10 & 0.153 & 0.188 & 1.225 \\
Samson & 27.146 & 28.284 & 1.042 \\
Topics & 6.025 & 6.042 & 1.003 \\
AffectiveText & 29.073 & 28.893 & 0.994 \\
UTKFace & 3.164 & 1.892 & 0.598 \\
PBMC & 42.265 & 39.982 & 0.946 \\
\bottomrule\end{tabular}\end{table*}

%% file: generated_synthetic_ci.tex
\begin{table}[!htbp]\centering\small
\caption{Paired max-disparity differences, method A minus B. Intervals are 95\% percentile bootstrap intervals from 10,000 resamples of 200 paired repetitions, without multiplicity adjustment. Positive differences favor B. Intervals containing zero do not establish equivalence.}\label{tab:paired-synthetic}
\setlength{\tabcolsep}{4pt}
\begin{tabular}{@{}lll@{}}\toprule
Task & A - B & Difference [95\% interval] \\\midrule
D2 & Global - FullCP & $+0.1548\;[+0.1503,+0.1593]$ \\
D2 & Global - TwoStage & $+0.1412\;[+0.1365,+0.1459]$ \\
D2 & Global - OneShot & $+0.1473\;[+0.1428,+0.1517]$ \\
D2 & Global - TrainRes & $+0.1561\;[+0.1512,+0.1611]$ \\
D2 & TrainRes - FullCP & $-0.0013\;[-0.0046,+0.0020]$ \\
D2 & TrainRes - Jackknife+ & $-0.0003\;[-0.0036,+0.0029]$ \\
D2 & Mondrian - FullCP & $+0.0030\;[-0.0001,+0.0060]$ \\
D2 & Mondrian - TwoStage & $-0.0107\;[-0.0141,-0.0074]$ \\
D3 & Global - Mondrian & $+0.0855\;[+0.0832,+0.0879]$ \\
D3 & FullCP - Mondrian & $+0.0765\;[+0.0741,+0.0788]$ \\
D3 & TwoStage - Mondrian & $+0.0844\;[+0.0815,+0.0874]$ \\
D5 & Global - FullCP & $+0.1041\;[+0.1003,+0.1079]$ \\
D5 & Global - TwoStage & $+0.0914\;[+0.0872,+0.0956]$ \\
D5 & Global - OneShot & $+0.0992\;[+0.0943,+0.1042]$ \\
D5 & Global - TrainRes & $+0.1060\;[+0.1016,+0.1103]$ \\
D5 & TrainRes - Jackknife+ & $-0.0016\;[-0.0044,+0.0010]$ \\
D5 & Mondrian - FullCP & $+0.0146\;[+0.0117,+0.0174]$ \\
D5 & Mondrian - TwoStage & $+0.0019\;[-0.0015,+0.0052]$ \\
D6 & Global - Mondrian & $+0.3596\;[+0.3562,+0.3631]$ \\
D6 & TwoStage - Mondrian & $+0.2330\;[+0.2296,+0.2364]$ \\
\bottomrule\end{tabular}\end{table}

%% file: generated_batch_full.tex
\begin{table}[!htbp]\centering\small
\caption{BatchMVP experiment, full mean metrics. All methods are paired within each task. The coverage screen is one-sided at 0.88; CIFAR-10 and Samson show substantial BatchMVP overcoverage.}\label{tab:batch-full}
\setlength{\tabcolsep}{4pt}
\begin{tabular}{@{}lllllll@{}}\toprule
Task & Method & Cov. & $D_g$ & $D_{\mathcal H}$ & Worst & Radius \\\midrule
D1 & Global & 0.901 & 0.020 & 0.023 & 0.891 & 0.431 \\
D1 & Mondrian & 0.907 & 0.038 & 0.039 & 0.879 & 0.438 \\
D1 & BatchMVP & 0.891 & 0.041 & 0.048 & 0.863 & 0.426 \\
D2 & Global & 0.903 & 0.198 & 0.253 & 0.702 & 0.518 \\
D2 & Mondrian & 0.917 & 0.052 & 0.069 & 0.880 & 0.499 \\
D2 & BatchMVP & 0.887 & 0.057 & 0.069 & 0.848 & 0.456 \\
D3 & Global & 0.901 & 0.109 & 0.112 & 0.805 & 1.056 \\
D3 & Mondrian & 0.905 & 0.023 & 0.032 & 0.893 & 0.742 \\
D3 & BatchMVP & 0.890 & 0.027 & 0.057 & 0.878 & 0.734 \\
D4 & Global & 0.903 & 0.135 & 0.140 & 0.767 & 0.805 \\
D4 & Mondrian & 0.914 & 0.050 & 0.071 & 0.877 & 0.791 \\
D4 & BatchMVP & 0.885 & 0.061 & 0.073 & 0.842 & 0.766 \\
D5 & Global & 0.902 & 0.139 & 0.171 & 0.762 & 0.449 \\
D5 & Mondrian & 0.913 & 0.050 & 0.056 & 0.876 & 0.456 \\
D5 & BatchMVP & 0.886 & 0.057 & 0.069 & 0.847 & 0.415 \\
D6 & Global & 0.901 & 0.381 & 0.381 & 0.519 & 2.165 \\
D6 & Mondrian & 0.901 & 0.021 & 0.037 & 0.885 & 1.776 \\
D6 & BatchMVP & 0.896 & 0.024 & 0.029 & 0.879 & 1.767 \\
CIFAR-10 & Global & 0.900 & 0.362 & 0.362 & 0.538 & 0.151 \\
CIFAR-10 & Mondrian & 0.903 & 0.022 & 0.047 & 0.885 & 0.188 \\
CIFAR-10 & BatchMVP & 0.953 & 0.099 & 0.099 & 0.889 & 0.188 \\
Topics & Global & 0.900 & 0.014 & 0.048 & 0.891 & 6.028 \\
Topics & Mondrian & 0.902 & 0.023 & 0.047 & 0.884 & 6.045 \\
Topics & BatchMVP & 0.897 & 0.028 & 0.039 & 0.878 & 5.985 \\
AffectiveText & Global & 0.905 & 0.122 & 0.128 & 0.788 & 29.095 \\
AffectiveText & Mondrian & 0.919 & 0.080 & 0.090 & 0.872 & 29.973 \\
AffectiveText & BatchMVP & 0.889 & 0.088 & 0.101 & 0.821 & 28.109 \\
Samson & Global & 0.901 & 0.248 & 0.248 & 0.652 & 27.145 \\
Samson & Mondrian & 0.902 & 0.022 & 0.062 & 0.885 & 27.975 \\
Samson & BatchMVP & 0.962 & 0.085 & 0.097 & 0.903 & 28.094 \\
PBMC & Global & 0.900 & 0.100 & 0.149 & 0.875 & 42.274 \\
PBMC & Mondrian & 0.901 & 0.017 & 0.085 & 0.893 & 39.973 \\
PBMC & BatchMVP & 0.900 & 0.021 & 0.053 & 0.888 & 38.389 \\
UTKFace & Global & 0.900 & 0.275 & 0.285 & 0.625 & 3.163 \\
UTKFace & Mondrian & 0.901 & 0.015 & 0.090 & 0.890 & 1.850 \\
UTKFace & BatchMVP & 0.919 & 0.080 & 0.080 & 0.874 & 1.562 \\
\bottomrule\end{tabular}\end{table}
\begin{table}[!htbp]\centering\small
\caption{BatchMVP minus baseline: 95\% paired bootstrap intervals (2,000 resamples). Negative differences favor BatchMVP. These are unadjusted exploratory comparisons.}\label{tab:batch-ci}
\setlength{\tabcolsep}{4pt}
\begin{tabular}{@{}llll@{}}\toprule
Task & Baseline & $\Delta D_g$ [95\% interval] & $\Delta D_{\mathcal H}$ [95\% interval] \\\midrule
D1 & Global & $+0.021\;[+0.018,+0.023]$ & $+0.026\;[+0.023,+0.028]$ \\
D1 & Mondrian & $+0.003\;[+0.000,+0.005]$ & $+0.010\;[+0.007,+0.012]$ \\
D2 & Global & $-0.141\;[-0.146,-0.136]$ & $-0.184\;[-0.190,-0.178]$ \\
D2 & Mondrian & $+0.005\;[+0.002,+0.009]$ & $-0.000\;[-0.004,+0.003]$ \\
D3 & Global & $-0.082\;[-0.084,-0.079]$ & $-0.055\;[-0.058,-0.051]$ \\
D3 & Mondrian & $+0.004\;[+0.002,+0.006]$ & $+0.026\;[+0.022,+0.030]$ \\
D4 & Global & $-0.074\;[-0.078,-0.069]$ & $-0.066\;[-0.071,-0.062]$ \\
D4 & Mondrian & $+0.011\;[+0.007,+0.014]$ & $+0.003\;[-0.001,+0.006]$ \\
D5 & Global & $-0.082\;[-0.087,-0.078]$ & $-0.102\;[-0.107,-0.097]$ \\
D5 & Mondrian & $+0.007\;[+0.004,+0.011]$ & $+0.012\;[+0.009,+0.016]$ \\
D6 & Global & $-0.357\;[-0.361,-0.354]$ & $-0.352\;[-0.355,-0.349]$ \\
D6 & Mondrian & $+0.003\;[+0.001,+0.004]$ & $-0.008\;[-0.009,-0.006]$ \\
CIFAR-10 & Global & $-0.262\;[-0.269,-0.255]$ & $-0.262\;[-0.270,-0.255]$ \\
CIFAR-10 & Mondrian & $+0.077\;[+0.075,+0.079]$ & $+0.052\;[+0.051,+0.054]$ \\
Topics & Global & $+0.013\;[+0.010,+0.016]$ & $-0.009\;[-0.013,-0.005]$ \\
Topics & Mondrian & $+0.005\;[+0.002,+0.007]$ & $-0.008\;[-0.012,-0.004]$ \\
AffectiveText & Global & $-0.034\;[-0.041,-0.026]$ & $-0.027\;[-0.034,-0.020]$ \\
AffectiveText & Mondrian & $+0.008\;[+0.002,+0.015]$ & $+0.011\;[+0.005,+0.017]$ \\
Samson & Global & $-0.163\;[-0.166,-0.160]$ & $-0.151\;[-0.153,-0.148]$ \\
Samson & Mondrian & $+0.064\;[+0.062,+0.066]$ & $+0.036\;[+0.033,+0.039]$ \\
PBMC & Global & $-0.079\;[-0.081,-0.078]$ & $-0.095\;[-0.099,-0.092]$ \\
PBMC & Mondrian & $+0.003\;[+0.002,+0.004]$ & $-0.032\;[-0.035,-0.029]$ \\
UTKFace & Global & $-0.196\;[-0.199,-0.192]$ & $-0.205\;[-0.208,-0.202]$ \\
UTKFace & Mondrian & $+0.065\;[+0.063,+0.066]$ & $-0.010\;[-0.013,-0.007]$ \\
\bottomrule\end{tabular}\end{table}

%% file: additional_map_sweep.tex
\begin{table}[!htbp]\centering\small
\caption{Supplementary released-array robustness analysis: Global (G) versus Mondrian (M), five requested groups per map. Negative paired differences favor M. Realized group counts can be smaller. This is a separate two-method experiment, not the all-method comparison in Table~\ref{tab:real-strata-sensitivity}.}\label{tab:additional-map-sweep}
\setlength{\tabcolsep}{4pt}
\begin{tabular}{@{}llllll@{}}\toprule
Task & Map & Groups & $D_G$ & $D_M$ & M - G [95\% interval] \\\midrule
CIFAR-10 & boundary & 5 & 0.259 & 0.020 & $-0.239\;[-0.245,-0.233]$ \\
CIFAR-10 & entropy & 5 & 0.358 & 0.020 & $-0.338\;[-0.345,-0.331]$ \\
CIFAR-10 & dominant & 5 & 0.049 & 0.021 & $-0.027\;[-0.030,-0.025]$ \\
CIFAR-10 & kmeans & 5 & 0.071 & 0.025 & $-0.046\;[-0.050,-0.041]$ \\
Topics & boundary & 5 & 0.045 & 0.025 & $-0.020\;[-0.024,-0.017]$ \\
Topics & entropy & 5 & 0.014 & 0.025 & $+0.011\;[+0.009,+0.013]$ \\
Topics & dominant & 5 & 0.039 & 0.028 & $-0.011\;[-0.015,-0.008]$ \\
Topics & kmeans & 5 & 0.034 & 0.032 & $-0.002\;[-0.005,+0.001]$ \\
AffectiveText & boundary & 5 & 0.120 & 0.082 & $-0.038\;[-0.046,-0.030]$ \\
AffectiveText & entropy & 5 & 0.102 & 0.065 & $-0.037\;[-0.043,-0.031]$ \\
AffectiveText & dominant & 5 & 0.106 & 0.104 & $-0.002\;[-0.007,+0.003]$ \\
AffectiveText & kmeans & 5 & 0.146 & 0.075 & $-0.071\;[-0.080,-0.062]$ \\
Samson & boundary & 5 & 0.249 & 0.022 & $-0.227\;[-0.231,-0.224]$ \\
Samson & entropy & 5 & 0.195 & 0.022 & $-0.173\;[-0.176,-0.170]$ \\
Samson & dominant & 2 & 0.084 & 0.011 & $-0.073\;[-0.075,-0.071]$ \\
Samson & kmeans & 5 & 0.900 & 0.079 & $-0.821\;[-0.827,-0.815]$ \\
PBMC & boundary & 2 & 0.100 & 0.017 & $-0.083\;[-0.084,-0.081]$ \\
PBMC & entropy & 5 & 0.145 & 0.029 & $-0.116\;[-0.120,-0.112]$ \\
PBMC & dominant & 5 & 0.077 & 0.032 & $-0.045\;[-0.048,-0.043]$ \\
PBMC & kmeans & 5 & 0.083 & 0.034 & $-0.049\;[-0.051,-0.047]$ \\
UTKFace & boundary & 5 & 0.280 & 0.013 & $-0.266\;[-0.270,-0.263]$ \\
UTKFace & entropy & 5 & 0.270 & 0.013 & $-0.257\;[-0.261,-0.253]$ \\
UTKFace & dominant & 5 & 0.091 & 0.029 & $-0.062\;[-0.067,-0.057]$ \\
UTKFace & kmeans & 5 & 0.165 & 0.017 & $-0.148\;[-0.151,-0.145]$ \\
\bottomrule\end{tabular}\end{table}

%% file: generated_strata_count.tex
\begin{table}[!htbp]\centering\small
\caption{Requested/realized group counts on released arrays. Maps are entropy for CIFAR-10, Topics and UTKFace, and boundary for the other tasks. Intervals use paired bootstrap resampling.}\label{tab:strata-count}
\setlength{\tabcolsep}{4pt}
\begin{tabular}{@{}llllll@{}}\toprule
Task & Req./real. & $D_G$ & $D_M$ & M - G [95\% interval] & Min. cal. \\\midrule
CIFAR-10 & 2/2 & 0.102 & 0.010 & $-0.092\;[-0.094,-0.090]$ & 1982.2 \\
CIFAR-10 & 3/3 & 0.188 & 0.014 & $-0.174\;[-0.179,-0.170]$ & 1310.0 \\
CIFAR-10 & 5/5 & 0.358 & 0.020 & $-0.338\;[-0.345,-0.331]$ & 774.7 \\
CIFAR-10 & 10/10 & 0.683 & 0.037 & $-0.646\;[-0.656,-0.635]$ & 377.6 \\
Topics & 2/2 & 0.007 & 0.010 & $+0.003\;[+0.002,+0.004]$ & 1493.9 \\
Topics & 3/3 & 0.010 & 0.017 & $+0.007\;[+0.006,+0.009]$ & 984.5 \\
Topics & 5/5 & 0.014 & 0.025 & $+0.011\;[+0.009,+0.013]$ & 580.0 \\
Topics & 10/10 & 0.028 & 0.039 & $+0.011\;[+0.009,+0.014]$ & 281.7 \\
AffectiveText & 2/2 & 0.026 & 0.030 & $+0.004\;[+0.002,+0.006]$ & 177.2 \\
AffectiveText & 3/3 & 0.034 & 0.044 & $+0.010\;[+0.008,+0.013]$ & 125.2 \\
AffectiveText & 5/5 & 0.120 & 0.082 & $-0.038\;[-0.046,-0.030]$ & 24.0 \\
AffectiveText & 10/9 & 0.132 & 0.099 & $-0.033\;[-0.040,-0.027]$ & 23.8 \\
Samson & 2/2 & 0.059 & 0.010 & $-0.050\;[-0.051,-0.049]$ & 1787.9 \\
Samson & 3/3 & 0.115 & 0.013 & $-0.102\;[-0.104,-0.099]$ & 1179.9 \\
Samson & 5/5 & 0.249 & 0.022 & $-0.227\;[-0.231,-0.224]$ & 698.2 \\
Samson & 10/10 & 0.430 & 0.038 & $-0.392\;[-0.397,-0.388]$ & 338.1 \\
PBMC & 2/1 & 0.007 & 0.007 & $+0.000\;[+0.000,+0.000]$ & 2000.0 \\
PBMC & 3/1 & 0.007 & 0.007 & $+0.000\;[+0.000,+0.000]$ & 2000.0 \\
PBMC & 5/2 & 0.100 & 0.017 & $-0.083\;[-0.084,-0.081]$ & 399.7 \\
PBMC & 10/3 & 0.100 & 0.031 & $-0.069\;[-0.071,-0.067]$ & 193.6 \\
UTKFace & 2/2 & 0.080 & 0.006 & $-0.074\;[-0.075,-0.073]$ & 4702.3 \\
UTKFace & 3/3 & 0.147 & 0.009 & $-0.137\;[-0.140,-0.134]$ & 2900.5 \\
UTKFace & 5/5 & 0.270 & 0.013 & $-0.257\;[-0.261,-0.253]$ & 1829.8 \\
UTKFace & 10/10 & 0.484 & 0.024 & $-0.461\;[-0.467,-0.455]$ & 887.3 \\
\bottomrule\end{tabular}\end{table}

%% file: generated_predictor_sweep.tex
\begin{table}[!htbp]\centering\small
\caption{Controlled D4 bias sweep, 200 paired repetitions. Maps are recomputed at each predictor level; the latent samples, noise, and split indices are paired.}\label{tab:predictor-sweep}
\setlength{\tabcolsep}{4pt}
\begin{tabular}{@{}lllll@{}}\toprule
Bias scale & Mean score & Global $D_g$ & Global worst & Mondrian $D_g$ \\\midrule
0.45 & 0.552 & 0.151 & 0.750 & 0.051 \\
0.3 & 0.383 & 0.141 & 0.761 & 0.051 \\
0.15 & 0.233 & 0.127 & 0.775 & 0.051 \\
0.0 & 0.163 & 0.128 & 0.773 & 0.050 \\
\bottomrule\end{tabular}\end{table}